\documentclass[ejsv2, noshowframe]{imsart}

\RequirePackage[numbers]{natbib}
\RequirePackage{graphicx}

\usepackage{url,amsmath,amsfonts}
\usepackage{hyperref}       
\hypersetup{colorlinks=true,
  linkcolor=blue,
  filecolor=magenta,
  urlcolor=cyan,
  pdfborderstyle={/S/U/W 1}}
\usepackage{caption}
\usepackage[justification=centering]{subfig}
\usepackage{algorithm}
\usepackage{algorithmic}
\usepackage{listings}
\usepackage[T1]{fontenc}
\usepackage{enumerate}   
\usepackage{color}
\usepackage{relsize}
\usepackage{multirow}
\usepackage[utf8]{inputenc}
\usepackage{booktabs}       
\usepackage{nicefrac}       
\usepackage{microtype}      
\usepackage[toc,page]{appendix}
\usepackage[dvipsnames]{xcolor}
\usepackage[normalem]{ulem}

\arxiv{2010.00000}
\startlocaldefs
\numberwithin{equation}{section}

\DeclareMathOperator*{\mL}{\mathcal L}
\DeclareMathOperator*{\mC}{\mathcal C}
\DeclareMathOperator*{\mD}{\mathcal D}
\DeclareMathOperator*{\mG}{\mathcal G}

\DeclareMathOperator*{\bmG}{\overline{\mathcal G}}

\DeclareMathOperator*{\calO}{\mathcal O}

\DeclareMathOperator*{\bP}{\mathbb P}

\DeclareMathOperator*{\bI}{\mathbb I}
\DeclareMathOperator*{\bE}{\mathbb E}

\DeclareMathOperator*{\bdS}{\mathbf{S}}
\DeclareMathOperator*{\bdvS}{\mathbf{\vec{S}}}

\DeclareMathOperator*{\bdvV}{\mathbf{\vec{V}}}

\DeclareMathOperator*{\bdY}{\mathbf{Y}}

\DeclareMathOperator*{\bdvY}{\mathbf{\vec{Y}}}
\DeclareMathOperator*{\bdpi}{\pmb{\pi}}

\DeclareMathOperator*{\tranp}{\text{T}}

\theoremstyle{plain}

\def \endprf{\hfill {\vrule height6pt width6pt depth0pt}\medskip}

\newcommand{\Method}{$\widehat{mLPR}$-}
\newcommand{\Objective}{CATCH}

\newtheorem{theorem}{Theorem}[section]

\newtheorem{proposition}[theorem]{Proposition}

\newtheorem{corollary}[theorem]{Corollary}

\endlocaldefs

\begin{document}

\begin{frontmatter}

\title{Ranking hierarchical multi-label classification results with mLPRs}  
\runtitle{mLPR}

\begin{aug}
\author[A]{\fnms{Yuting} \snm{Ye}\ead[label=e1]{yeyuting@wizardquant.com}},
\author[B]{\fnms{Christine} \snm{Ho}\ead[label=e2]{christine.ho@siriusxm.com}}
\author[C]{\fnms{Ci-Ren} \snm{Jiang}\ead[label=e3]{cirenjiang@ntu.edu.tw}}
\author[D]{\fnms{Wayne Tai} \snm{Lee}\ead[label=e4]{wayne.lee@dixide.net}}
\author[E]{\fnms{Haiyan} \snm{Huang}\thanksref{t1}\ead[label=e5]{hhuang@stat.berkeley.edu}}
\thankstext{t1}{Corresponding author}
\address[A]{Wizard Quant, Shanghai, China.\printead[presep={,\ }]{e1}}
\address[B]{SiriusXM/Pandora, NY, USA.\printead[presep={,\ }]{e2}}
\address[C]{Institute of Statistics and Data Science, National Taiwan University, Taipei, Taiwan. \printead[presep={,\ }]{e3}}
\address[D]{Dixide, Taiwan. \printead[presep={,\ }]{e4}}
\address[E]{Department of Statistics University of California Berkeley, CA, USA.\printead[presep={,\ }]{e5}}
\runauthor{Ye and et al.}
\end{aug}

\begin{abstract}
Hierarchical multi-label classification (HMC) has gained considerable attention in recent decades. A seminal line of HMC research addresses the problem in two stages: first, training individual classifiers for each class, then integrating these classifiers to provide a unified set of classification results across classes while respecting the given hierarchy. In this article, we focus on the less attended second-stage question while adhering to the given class hierarchy. 
This involves addressing a key challenge: how to manage the hierarchical constraint and account for statistical differences in the first-stage classifier scores across different classes to make classification decisions that are optimal under a justifiable criterion.
{\color{black} To address this challenge, we introduce a new objective function, called \Objective\/, to ensure reasonable classification performance. To optimize this function, we propose a decision strategy built on a novel metric, the multidimensional Local Precision Rate (mLPR), which reflects the membership chance of an object in a class given all classifier scores and the class hierarchy.}
Particularly, we demonstrate that, under certain conditions, transforming the classifier scores into mLPRs and comparing mLPR values for all objects against all classes can, in theory, ensure the class hierarchy and maximize \Objective\/. In practice, we propose an algorithm HierRank to rank estimated mLPRs under the hierarchical constraint, leading to a ranking that maximizes an empirical version of \Objective\/. Our approach was evaluated on a synthetic dataset and two real datasets, exhibiting superior performance compared to several state-of-the-art methods in terms of improved decision accuracy. 
\end{abstract}




\end{frontmatter}

\section{Introduction}\label{sec:intro}
Hierarchical multi-label classification (HMC) is a task that requires incorporating  additional knowledge of the dependency relationships between classes along with the multi-label classification of each object into one or more classes \citep{zhang2013review}. The hierarchical class dependency in HMC is generally represented by a tree or a directed acyclic graph (DAG). Recently, there has been considerable interest in the field of statistics and machine learning in HMC, which is a crucial problem encountered in many applications. In biology and biomedicine, HMC applications include the diagnosis of diseases along a DAG composed of terms from the Unified Medical Language System (UMLS) \citep{huang2010, jiang2014}, the assignment of genes to multiple gene functional categories defined by the Gene Ontology DAG \citep{alves2010, feng2017hierarchical, kahanda2017gostruct}, the categorization of MIPS FunCat rooted tree \citep{valentini2011}, and numerous other biomedical examples \citep{chen2019deep, makrodimitris2019improving, nakano2019machine, pham2021interpreting}. Outside of biology, HMC is commonly used in text classification, music categorization, and image recognition \citep{gupta2016product, salama2016semantic, zeng2017knowledge, yang2018visually}.

Specifically, suppose there is a random object and $K$ classes that are organized in a hierarchical structure. 
{\color{black} The goal of HMC is to determine which of the $K$ classes this random object belongs to}. One special condition of HMC requires that this object must be a member of all the parent and ancestor classes in the hierarchy, if it belongs to a child class. A seminal line of HMC research has tackled the problem through a two-stage approach. In the first stage, classifiers are trained for each of the $K$ classes without considering the class hierarchy, as if there were multiple independent classification problems. In the second stage, the task is to render a decision for each object regarding each class based on the class hierarchy and a predefined performance criterion, using the classifier scores from the first stage \citep{silla2011survey, feng2017hierarchical, chen2019deep}. This approach is popular due to its flexibility and computational efficiency, as various classification methods can be applied in the first stage, and the class-specific classifiers can be trained in parallel. However, it remains an open question at the second stage how to effectively integrate and refine the initial classification results from the first stage so that the refined classification results will 1) respect the given class hierarchy and 2) achieve the best possible classification performance as evaluated by a statistically meaningful performance measure.

 In this two-stage approach, some methods make decisions based on class-specific cutoffs for the first-stage classifier scores, which are determined by optimizing objectives such as H-loss or F-measure \citep{barutcuoglu2006, triguero2016labelling}. As the cutoff values are determined without taking into account the hierarchical structure, the class decisions may not respect the hierarchy. 
{\color{black} While a remedial procedure can be employed to enforce compliance with the hierarchical constraint, it no longer guarantees optimal performance for the adjusted decisions \citep{ananpiriyakul2014label}.}

Many other methods proceed with the problem by sorting the objects against all classes, given the class hierarchy and the classifier scores of the objects for every class. In this treatment, a single cutoff on the ranking suffices to produce all decisions. For instance, \citet{jiang2014} proposed an optimal ranking method for the general multi-label setting: transforming the first-stage classifier scores to local precision rates and then ordering them in descending order. The resulting ranking maximizes the pooled precision rate at any pooled recall rate. The local precision rate concept is also attractive in that it lends itself to a Bayesian interpretation, and its value is equivalent to the local true discovery rate ($\ell$tdr)\footnote{True discovery rate equals one minus false discovery rate, where the definition of ``discovery'' can be found in Section 4.1. of \citet{efron2012}; see \citet{efron2012} for a detailed discussion about true discovery rate.} under certain probabilistic assumptions. However, this method does not consider the hierarchical structure.

There are also efforts that do not directly generate a complete ranking, but achieve decision-making under HMC by optimizing a predefined objective function when respecting the class hierarchy, given the first-stage classifiers. \citet{bi2011} proposed an algorithm that maximizes the sum of the top $L$ first-stage classifier scores while respecting the hierarchy, where $L$ is predefined. Nonetheless, directly summing these classifier scores across different classes is potentially problematic because the classifiers for different classes may have {\color{black} very} different statistical properties, which could lead to suboptimal decisions if not properly handled. \citet{bi2015bayes} extended \citet{bi2011} by introducing an algorithm that optimizes some objective function (instead of the sum of the top $L$ classifier scores) under the hierarchical constraint. Although they explored three possible objective/risk functions, they did not specify which one to use, nor did they provide a clear statistical justification or interpretation for the three functions or their hyperparameters.

{\color{black} Additionally, there has been growing interest in using neural network approaches for HMC in recent years. Wehrmann et al. \cite{wehrmann2017hierarchical, wehrmann2018hierarchical} proposed neural network architectures capable of simultaneously optimizing both local and global loss functions. These architectures aim to discover local hierarchical class relationships and extract global information from the entire class hierarchy, while also penalizing hierarchical violations. Giunchiglia and Lukasiewicz \cite{giunchiglia2020neurips, giunchiglia2021multi} introduced a neural network approach that leverages hierarchical information to generate predictions consistent with hierarchy constraints, thereby improving performance. Although delivering excellent performance, these neural network results do not provide direct statistical interpretations.}

In this article, we assume that ``good'' classifiers for individual classes are already given, and we aim to develop a strategy for optimal decision-making at the second stage by integrating the initial classifier scores (from the first stage) under an HMC framework. A key challenge is how to handle the hierarchical constraint and  the statistical differences of the given classifier scores among different classes in one unified model. If not accounted for properly, such differences can lead to poor classification decisions on some classes; see Supplement D.5 for a detailed discussion. To address this challenge, we develop a strategy based on a new quantity, called the multidimensional Local Precision Rate (mLPR), for each object in each class, which can be derived based on the first-stage classifier scores and the given class hierarchy. Under certain conditions, we show that transforming the classifier scores into mLPRs and comparing mLPR values for all objects against all classes can, in theory, ensure the hierarchical consistency and maximize a new objective function that is related to the area under a hit curve. We refer to this new objective function as the Conditional expected Area under The Curve of Hit (\Objective\/). 
{\color{black} In our context, an object labeled positive for a class indicates that it is predicted to belong to that class, whereas a negative label indicates the opposite.  {\color{black} The piecewise hit curve, $h(x)$, is then defined as the number of correctly predicted positives among the first $r=\lceil x \rceil$ predicted positive labels.} To generate predictions, we sort all classification cases for all objects across all classes in descending order of their mLPRs. The top $r$ in this ranking are labeled positive, and the remainder are labeled negative.} 
This hit curve and \Objective\/ provide an intuitive representation of classification performance for a given ranking for classification decisions.

Despite the theoretical advantages of the above mLPR-based ranking and decision making strategy, applying this approach in practice faces a new challenge: {\color{black}the quantities that are required to compute mLPRs need to be estimated from the data;} sorting the estimated mLPRs in descending order might fail to guarantee the optimization of \Objective\/ or may violate the hierarchy constraint. 
To overcome this challenge, we introduce a new algorithm, HierRank, which, {\color{black}when applied to the estimated mLPRs,} produces a ranking of all objects against all classes by maximizing an empirical version of \Objective\/ (defined based on the estimated mLPRs) while respecting the class hierarchy constraint. Additionally, we propose a cutoff selection procedure 
on the resulting ranking to control certain statistical properties of the final classification results, such as controlling the false discovery rate (FDR) at a target level or achieving the maximum $F_1$ score.

We compared our method to several state-of-the-art HMC methods using a synthetic dataset and two real datasets. 
{\color{black} Our approach demonstrated superior performance in decision-making, particularly at the start of the precision-recall curve where precision is prioritized, while performing at least comparably throughout the remainder of the curve.}

To summarize, our main contributions in this article {\color{black} are}: (i) We introduce a new approach for decision-making in HMC based on the newly introduced quantity mLPR and the objective function \Objective\/. We present desirable theoretical statistical properties of this strategy. (ii) We suggest several methods to estimate mLPRs in real-world application scenarios, {\color{black} and investigate the impact of the mLPR estimation on our method's performance. (iii) We introduce HierRank, {\color{black}a new algorithm that, when applied to the estimated mLPRs,} ranks all objects across all classes by maximizing an empirical version of \Objective\/ under the class-hierarchy constraint. (iv)} We compare our method to several state-of-the-art HMC methods using a synthetic dataset and two real datasets, demonstrating that our approach outperforms, 
{\color{black}or at least matches, these methods across several evaluation metrics, including precision rates, recall rates, and {\color{black}$F_1$} scores.}

The rest of the paper is organized as follows. In Section \ref{sec:notation_model}, we introduce the notation and our model. In Section \ref{sec:obj_func}, we present the objective function \Objective\/ and the quantity mLPR, and show that sorting mLPRs in descending order can maximize \Objective\/ while respecting the hierarchy. In Section \ref{sec:methods}, we propose the ranking algorithm HierRank, which can sort objects by estimated mLPRs under the hierarchy constraint. Section \ref{sec:evaluation} reports the performance of our approach in comparison with a few other methods on a synthetic dataset and two case studies. Finally, we conclude the article in Section \ref{sec:hmc_conclusion}. Supplementary information can be found at the end of the article.

\section{Notation and Model}\label{sec:notation_model}
Suppose there are $K$ classes that are structured in a known hierarchical graph $\mG$, {\color{black} which is assumed to be a single tree, or a set of disjoint trees (e.g., Figure \ref{fig:model_graph}~(a)), or a directed acyclic graph (DAG). The difference between a tree and a DAG is that each node in a tree has at most one parent, whereas a node in a DAG can have multiple parents.}  In $\mG$, $pa(k)$ denotes the set of the parent nodes of the node $k$, $anc(k)$ denotes the set of its ancestor nodes, and $nbh(k)$ denotes the set of its parents and direct children. In the example shown in {\color{black} Figure \ref{fig:model_graph}~(a)}, $pa(\text{F}) = \{\text{C}\}$, $anc(\text{F}) = \{\text{A}, \text{C}\}$, and $nbh(\text{F}) = \{\text{C}, \text{G}, \text{I}\}$. {\color{black} In this article, we focus mainly on the tree structure; the extension of the results to the DAG structure is discussed in Supplement B.4.}

\begin{figure}[tp]
  \centering
  \begin{minipage}{0.49\linewidth}
    \centering
    \includegraphics[width=0.7\linewidth]{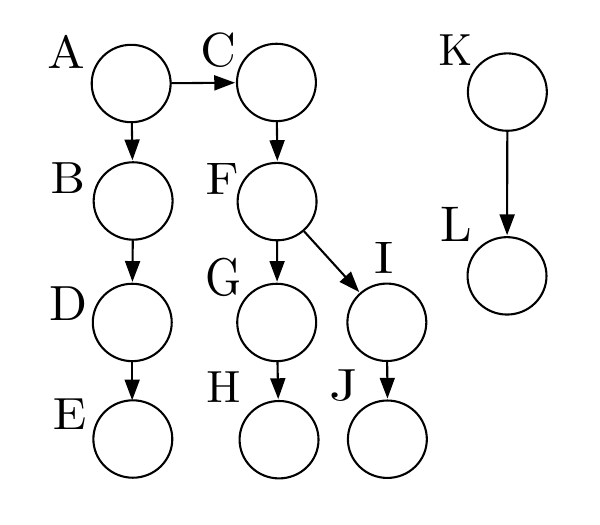}
   \subfloat[An example hierarchy $\mG$]{\hspace{.55\linewidth}}    
  \end{minipage}
  \begin{minipage}{0.49\linewidth}
      \centering
      \includegraphics[width=0.7\linewidth]{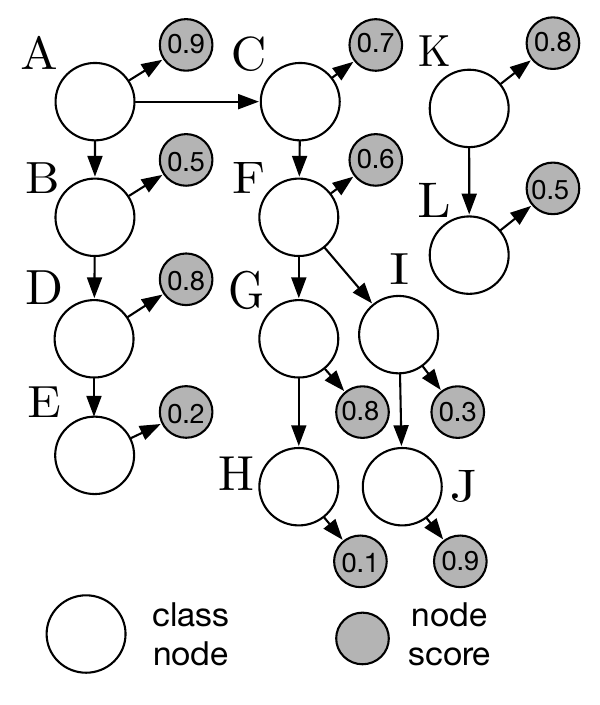}
  \subfloat[The augmented graph $\bmG$]{\hspace{.55\linewidth}}      
  \end{minipage}
  \caption{Example hierarchical graph $\mG$ and its associated augmented graph $\bmG$.}
  \label{fig:model_graph}
\end{figure}

{\color{black} 
For the $K$ nodes/classes in $\mG$,
$K$ classification decisions need to be made for each object. The membership status of an object in class $k$ ($k=1,...,K$) is denoted by $Y_k$, where $Y_k=1$ indicates that the object belongs to the class, and $Y_k=0$ indicates otherwise. Additionally, each object is associated with a node classification score $S_k$ for class $k$ (e.g., $S_k$ may represent the pre-given first-stage classifier score). As noted in Section \ref{sec:intro}, in the two-stage approach to HMC problems, a separate classifier is first trained for each of the $K$ classes, ignoring the class hierarchy. When applied to an object, each classifier then produces a score $S_k$, which is a statistic serving as evidence of that object’s membership in class $k$. The membership labels of the object across the $K$ classes are represented as $\bdY:=(Y_1, \ldots, Y_K)^{\tranp}$, and the classification scores for the $K$ classes are represented by $\bdS:=(S_1, \ldots, S_K)^{\tranp}$. Accordingly, we consider an augmented graph $\bmG$ (e.g., Figure \ref{fig:model_graph}~(b)) of $\mG$.}


{\color{black} For a random object, we assume $Y_k$ is a random binary variable with a membership probability of $\bP(Y_k=1)$, and that $S_k$ has the cumulative distribution function (CDF) $F_{0,k}$ {\color{black} given} $Y_k=0$ and $F_{1,k}$ {\color{black} given} $Y_k=1$. 
Additionally, we assume {\bf conditional independence} between $S_1,...,S_K$ and $Y_1,...,Y_K$. That is, 
given $Y_k$, we assume that $S_k$ is independent of other scores and other labels, or equivalently,}
\begin{itemize}
\item [\textup{(i)}] {\color{black} $\bP (S_k = s| S_1, \ldots, S_{k-1}, S_{k+1}, \ldots, S_K, Y_1, \ldots, Y_K) = \bP (S_k = s| Y_k)$.}  
\end{itemize}
{\color{black} This assumption is reasonable, particularly within the two-stage HMC framework, where the first-stage classifiers are trained separately for each class. It is also a standard assumption in Bayesian networks and is commonly used in HMC \citep{cesa2006hierarchical, bi2015bayes} to simplify the model and ensure computational tractability, though its appropriateness depends on the specific context and the methods employed.} {\color{black} In the disease-diagnosis setting, the assumption is especially reasonable when each disease class has independent, representative training samples and the key features used by each classifier do not overlap. For example, it is plausible for distantly related diseases, such as Parkinson’s disease and type II diabetes, whose underlying biological markers and clinical presentations share little in common.}

    
{\color{black} Note that the class hierarchical constraint dictates that if an object has a negative membership in a class/node, it must also have negative memberships in all of that class/node's descendants. That is, }
\begin{itemize}
\item [(ii)] $\bP (Y_k = 1| Y_{pa(k)} = 0) = 0$.
\end{itemize}

{\color{black}We'd also like to note the below property based on the earlier distributional assumptions on $S_k$:}
\begin{itemize}
 \item [(iii)] $\bP(S_k \leq s) = \bP(S_k \leq s, Y_k=0) + \bP(S_k \leq s, Y_k=1) = F_{0,k}(s) \bP (Y_k = 0)+ F_{1,k}(s) \bP (Y_k = 1)$, where $F_{0,k}$ denotes the null CDF $\bP(S_k \leq s|Y_k=0)$, and $F_{1,k}$ denotes the alternative CDF $\bP(S_k \leq s|Y_k=1)$.
\end{itemize}

{\color{black}When observing $M$ objects, we denote the class membership and associated classification score of object $m$ for class $k$ by $Y_{k}^{(m)}$ and $S_{k}^{(m)}$, respectively. Additionally, we represent $\bdY^{(m)}$ as $(Y_{1}^{(m)}, \ldots, Y_{K}^{(m)})^{\tranp}$, and $\bdS^{(m)}$ as $(S_{1}^{(m)}, \ldots, S_{K}^{(m)})^{\tranp}$. For simplicity, we vectorize $(\bdY^{(1)}, \ldots, \bdY^{(M)})$ and $(\bdS^{(1)}, \ldots, \bdS^{(M)})$ to obtain} $$\bdvY = (Y_1^{(1)}, \ldots, Y_K^{(1)}, \ldots, Y_1^{(M)}, \ldots, Y_K^{(M)})^{\tranp}$$ and $$\bdvS = (S_1^{(1)}, \ldots, S_K^{(1)}, \ldots, S_1^{(M)}, \ldots, S_K^{(M)})^{\tranp},$$ respectively. Letting $i = (m - 1) \cdot K + k$, where $m$ denotes the object index and $k$ the class index, we redefine $Y_i := Y_k^{(m)}$, leading to $\bdvY = (Y_1, \ldots, Y_n)^{\tranp}$ and $\bdvS = (S_1, \ldots, S_n)^{\tranp}$ with $n=M\times K$, when there is no ambiguity.

{\color{black} We refer to a situation where $Y_i = 1$ as ``positive Event $i$'' (or simply ``Event $i$'' when there is no risk of confusion). For ease of notation, throughout the paper, we define that Event $i$, where $i = (m-1)\cdot K + k$, is an ancestor of Event $i'$ (or, equivalently, $i \in anc(i')$) if both events $i$ and $i'$ concern the same object $m$ and the node/class associated with Event $i$ is an ancestor node/class of that of Event $i'$. } 
{\color{black} We define a ranking $\bdpi = (\pi_1, \ldots, \pi_n)$ on $n$ events as a permutation of \{$1,\ldots, n$\}. Here, $\pi_i$, a simplified notation for $\pi(i)$, denotes the rank assigned to event $i$, {\bf with no ties}, 
and $\pi^{-1}(r)$ represents the index of the event with rank $r$.} A ranking $\bdpi$ is said to have \textbf{hierarchical consistency}, or to be a \textbf{topological ordering} for $\mG$ if it satisfies 
\begin{equation*}
  \pi_i < \pi_{i'} \text{~ for any Event $i$ that is an ancestor of Event $i'$.}\label{eq:hierarchy_property}
\end{equation*}

\section{A Ranking Strategy based on the Multidimensional Local Precision Rates (mLPRs)}\label{sec:obj_func}
Just to repeat, we need to decide the class labels for $M$ objects in $K$ classes, where the $K$ classes are organized in a tree or a set of disjoint trees. This leads to a total of $n = M \times K$ (dependent) classification events for us to consider. 


\subsection{A New Objective Function: Conditional expected Area under The Curve of Hit (\Objective\/)}\label{sec:obj_mLPR}

We leverage the area under the hit curve to formulate an objective function. In our framework, we define the {\color{black} piecewise} hit curve for a ranking $\bdpi$ of $n$ events as the non-decreasing function 
{\color{black}
$$h_{\bdpi}: (0,n] \rightarrow \{0, ..., n\}, \quad
h_{\bdpi}(x) = \sum_{j=1}^n \bI (Y_{\pi^{-1}(j)}=1) \bI( j \leq \lceil x \rceil ),$$ for any $x\in(0,n]\subset \mathbb{R}$.} Here, $Y_i \in \{1, 0\}$ is the true label of event $i$. We abbreviate $h_{\bdpi}$ as $h$ when $\bdpi$ is clear from context. By predicting the top $r$ events in the ranking $\bdpi$ as positive and the remaining $n - r$ events as negative, {\color{black}the value of $h(x)$ for $x\in(r-1,r]$ gives the number of correctly predicted positive labels among the first $r$ predicted positives. As $x$ ranges over $(0,n]$, $h(x)$ provides an intuitive summary of the overall classification performance of $\bdpi$.}
 
 Note that a hit curve can be easily converted into a precision–recall (PR) curve if the total number of hits, or true positives ($q$), is known: $\text{precision} = h(r)/r$; $\text{recall} = h(r)/q$. Another desirable property of {\color{black} the} hit curve is that they are more sensitive to distinguishing between different rankings for assessing their classification performance when the number of true positives is tiny relative to the total number of decisions to be made \citep{davis2006, herskovic2007, hand2009}. Figure S1 in Supplement A provides an example of a hit curve.

Among all of the potential rankings that respect the class hierarchy, we aim to find the one that maximizes the area under the hit curve. The area under the hit curve $h(x)$ when $x \in (r-1, r]$ is the number of true positives among the first $r$ predicted positives. Therefore, given a ranking $\pmb{\pi}$ of the $n$ events being considered, the total area under the corresponding hit curve (AUHC) can be expressed as
\begin{equation}\label{eq:AUC_hit}
  \text{AUHC}(\pmb{\pi}; \bdvY) = \sum_{r=1}^n \sum_{j=1}^r \bI (Y_{\pi^{-1}(j)} = 1) = \sum_{r=1}^n (n-r+1)\bI(Y_{\pi^{-1}(r)} = 1 ),
\end{equation}
where $\pmb{\pi}$ refers to a ranking of all of the $n$ events considered, as defined in Section \ref{sec:notation_model}. We consider the population mean of the area under the hit curve. Specifically, given the classifier scores $\bdvS$ and the class hierarchy in $\mG$, we compute the conditional expected values of \eqref{eq:AUC_hit}, and we obtain

\begin{eqnarray}
  \text{\Objective\/}(\pmb{\pi}; \bdvS, \mG) &:=& \bE ({\color{black} \text{AUHC}}(\pmb{\pi}; \bdvY) | \bdvS, \mG )\nonumber\\
  &=& \sum_{r=1}^n (n-r+1) \bP(Y_{\pi^{-1}(r)} = 1 | S_1, \ldots, S_n, \mG)
  . \label{eq:cobj}
\end{eqnarray}
This is our proposed objective function, the \textbf{Conditional expected Area under The Curve of Hit (\Objective\/)}. 

\subsection{The Multidimensional Local Precision Rate and its Properties}\label{sec:mLPR_property}
{\color{black} We denote the random variable introduced in Eq. \eqref{eq:cobj} as \textbf{multidimensional local precision rate (mLPR)}. That is, for Event $i$, 
$$mLPR_i := \bP(Y_i = 1 | S_1, \ldots, S_n, \mG).$$
}For simplicity, when no confusion arises, we omit $\pmb{\pi}$, $\mG$, or $\bdvS$ in $\text{\Objective}(\pmb{\pi}; \bdvS, \mG)$ and $\mG$ in $\bP(Y_i = 1 | S_1, \ldots, S_n, \mG)$. We now rewrite Eq. \eqref{eq:cobj} as $$\Objective\/ = \sum_{r=1}^n (n-r+1) mLPR_{\pi^{-1}(r)}.$$


The mLPR quantity extends the local precision rate \citep{jiang2014} by considering the class hierarchy in addition to the differences among classes. Moreover, it lends itself to a Bayesian interpretation: under certain probabilistic assumptions, it is equivalent to the multidimensional local true discovery rate (m$\ell$tdr) used in hierarchical hypothesis testing \citep{ploner2006multidimensional}. In this subsection, we will discuss the desirable properties of the mLPR quantity within the HMC framework.


\begin{proposition}\label{prop:hierarchy_consistency}
 Under the hierarchical constraint, for two events $i$ and $i'$, if $i \in anc(i')$, then $mLPR_i \geq mLPR_{i'}$. 
\end{proposition}
\begin{proof}
  By the definition of $anc(\cdot)$ in Section \ref{sec:notation_model}, if $i \in anc(i')$, then 
  the two events $i$ and $i'$ concern the same object and the associated class node of Event $i$ is an ancestor of that of Event $i'$. It follows that
  \begin{eqnarray*}
   && mLPR_{i'} \nonumber\\
   &=& \bP(Y_{i'} =1 | S_1, \cdots, S_n)\nonumber\\
    &=&  {\color{black} \sum_{\substack{\text{$y_j \in \{0,1\}$,}\\ \text{$\forall j\in [n]/\{i'\}$}}} \bP(Y_1=y_1, \cdots, Y_{i'} = 1, \cdots, Y_n=y_n | S_1, \cdots, S_n)}\nonumber\\
    &\overset{(a)}{=}& {\color{black} \sum_{\substack{\text{$y_j \in \{0,1\}$,}\\ \text{$\forall j\in [n]/\{i', i\}$}}} \bP(Y_1=y_1, \cdots, Y_{i} = 1, \cdots, Y_{i'} = 1, \cdots, Y_n=y_n | S_1, \cdots, S_n)}\nonumber\\
    &\leq& {\color{black}\sum_{\substack{\text{$y_j \in \{0,1\}$,}\\ \text{$\forall j\in [n]/\{i\}$}}} \bP(Y_1=y_1, \cdots, Y_{i} = 1,\cdots, Y_n=y_n| S_1, \cdots, S_n)}\nonumber\\
    &=& mLPR_{i},
  \end{eqnarray*}
where $(a)$ is obtained by the hierarchical consistency mathematically defined as Property (ii) in Section \ref{sec:notation_model}.
\end{proof}

Proposition \ref{prop:hierarchy_consistency} {\color{black} shows} that the mLPR value of an event cannot be less than those of its descendants. {\color{black} Proposition \ref{prop:decision_consistency} below further states that a positive event with a higher mLPR is more likely to occur than one with a lower mLPR, thus providing a statistical justification for directly comparing mLPR values to rank the events. Note that all the probabilities below are defined in accordance with the class hierarchy.
}

\begin{proposition}\label{prop:decision_consistency}
For any events $i$ and $i'$ with  $mLPR_i \geq mLPR_{i'}$, we have
  $$\bP (Y_i = 1| mLPR_i \geq mLPR_{i'}) \geq \bP(Y_{i'} = 1 | mLPR_i \geq mLPR_{i'}).$$
\end{proposition}
\begin{proof}
By definition, $mLPR_i = \bP(Y_i=1 | \bdvS)$. Then, the desired result follows from
\begin{eqnarray*}
  && \bP(Y_i = 1, mLPR_i \geq mLPR_{i'})\\
   &=& {\color{black}  \bE [\mathbb{I}(Y_i = 1) \mathbb{I}( mLPR_i \geq mLPR_{i'})]}\\
  &\overset{(a)}{=}& {\color{black} \bE \{\bE [\mathbb{I}(Y_i = 1) \mathbb{I}( mLPR_i \geq mLPR_{i'})| \bdvS]\}}\\
  &\overset{(b)}{=}& {\color{black} \bE \{\bE [\mathbb{I}(Y_i = 1)| \bdvS] \mathbb{I}( mLPR_i \geq mLPR_{i'})\}}\\
  &=& {\color{black} \bE \{mLPR_i \cdot \mathbb{I}( mLPR_i \geq mLPR_{i'})\}}\\  
  &\geq& {\color{black} \bE  \{mLPR_{i'} \cdot \mathbb{I}( mLPR_i \geq mLPR_{i'})\}}\\  
  &=& {\color{black} \bE \{\bE [\mathbb{I}(Y_{i'} = 1)| \bdvS] \mathbb{I}( mLPR_i \geq mLPR_{i'})\}}\\  
  &=&\bP(Y_{i'} = 1, mLPR_i \geq mLPR_{i'}),
\end{eqnarray*}
{\color{black} where (a) holds by law of total expectation and (b) holds because $mLPR_i$ and $mLPR_{i'}$ are functions of $\bdvS$.}
\end{proof}

Proposition \ref{prop:hierarchy_consistency} implies that mLPRs can ensure hierarchical consistency, and Proposition \ref{prop:decision_consistency} shows that events with higher mLPRs are more likely to be true positives. Together, these propositions suggest that sorting mLPRs can achieve effective classification results while adhering to the given hierarchy.

\subsection{A Ranking Strategy Based on mLPRs}\label{sec:mLPR_sorting}
We first consider a simple strategy, called \textbf{NaiveSort}, {\color{black} which 
sorts events by their associated values (e.g., classification scores, mLPRs, estimated mLPRs) in descending order.} 
{\color{black}When the values are tied, \textbf{NaiveSort} ranks the events according to the class hierarchy. 

\begin{corollary}\label{corr:naive_mlpr}
 A ranking $\pmb{\pi}$, produced by \textbf{NaiveSort} based on mLPRs, or equivalently by sorting mLPRs in descending order to rank the events, is a solution to the following constrained optimization problem. 
  \begin{eqnarray}
  \max_{\bdpi} && \Objective\/(\pmb{\pi}; \bdvS, \mG) \label{opt:max_cobj}\\
  \text{s.t.} && \text{$\pmb{\pi}$ is hierarchically consistent}.\nonumber
\end{eqnarray}
\end{corollary}
\begin{proof}
  {\color{black} 
  Given $CATCH = \sum_{r=1}^n (n-r+1)mLPR_{\pi^{-1}(r)}$, where $\pi^{-1}(r)$ represents the index of the event ranked $r$ as defined in earlier sections, $CATCH$ can be maximized if $\pmb{\pi}$ represents the ranking of events in descending order of their $mLPRs$. Specifically, the event with the largest $mLPR$ is assigned the highest rank ($r=1$), receiving the largest weight of $n$, while the event with the second-largest $mLPR$ is ranked second, receiving the next largest weight of $n-1$, and so on. 
  This process ensures that the weights are optimally aligned with the $mLPR$ values, thereby maximizing $CATCH$. Therefore, \textbf{NaiveSort}, which sorts events by $mLPR$ in descending order, achieves this maximization. 
  
  Next we show that the ranking $\pmb{\pi}$ produced by \textbf{NaiveSort} {\color{black}based on $mLPR$s} is hierarchically consistent. Specifically, we need to prove that for two events $i$ and $i'$, if $i \in anc(i')$, then $\pi_i < \pi_i'$. We will prove this using the method of contradiction. 
  
  Assume $i \in anc(i')$, but  $\pi_i > \pi_i'$.
  According to the definition of \textbf{NaiveSort}, this implies $mLPR_i \leq mLPR_i'$. Now, consider the two possible cases: 
  \begin{itemize}
      \item When $mLPR_i < mLPR_i'$: By Proposition \ref{prop:hierarchy_consistency}, $i$ cannot be an ancestor of $i'$, which contradicts the assumption $i \in anc(i')$. 
      \item When $mLPR_i = mLPR_i'$: In this case, \textbf{NaiveSort} resolves ties by ranking $i$ and $i'$ according to the class hierarchy. Therefore, if $\pi_i > \pi_i'$, $i$ cannot be an ancestor of $i'$, again contradicting $i \in anc(i')$. 
  \end{itemize}
 Thus, in both cases, we arrive at a contradiction. Therefore, the ranking given by \textbf{NaiveSort} {\color{black} based on $mLPR$s} is hierarchically consistent, which completes the proof.}
\end{proof}

\section{A Ranking Algorithm Based on  Estimated mLPRs}\label{sec:methods}

\textbf{NaiveSort} solves the constrained optimization problem in Eq. \eqref{opt:max_cobj},  when true mLPR values are available. In practice, however, mLPRs must be estimated from data. Using \textbf{NaiveSort} with estimated mLPRs does not guarantee adherence to the class hierarchy, let alone the maximization of \Objective. To address this, we propose a ranking method that{\color{black}, when applied to the estimated mLPRs,} maximizes an empirical approximation of \Objective\ while maintaining consistency with the hierarchy.
Before introducing the detailed algorithm in Section \ref{sec:algo}, we first explain in Section \ref{sec:mLPR_estimation} how mLPRs are estimated using the given classifier scores and class hierarchy graph. 

\subsection{Estimation of mLPRs}\label{sec:mLPR_estimation}

We start with computing $\bP(Y_1, \cdots, Y_n | S_1, \cdots, S_n)$:
\begin{eqnarray}
  \bP(Y_1 , \cdots, Y_n | S_1, \cdots, S_n) &\overset{(a)}{\propto}& \bP(S_1, \cdots, S_n | Y_1, \cdots, Y_n)  \cdot \bP(Y_1, \cdots, Y_n)\nonumber\\
  &\overset{(b)}{=}& \prod_{i=1}^n \bP(S_i | Y_i) \bP (Y_i | Y_{pa(i)})\nonumber\\
  &\overset{(c)}{\propto}& \prod_{i=1}^n \bP(Y_i | S_i) \cdot \frac{\bP(Y_i | Y_{pa(i)})}{\bP(Y_i)},\label{eq:mLPR}
\end{eqnarray}
Here, $(a)$ and $(c)$ follow from Bayes' rule, while $(b)$ is derived using the Markov property and the conditional independence outlined in Section \ref{sec:notation_model}.
If we have estimates of $\bP(S_i| Y_i)$ and $\bP(Y_i | Y_{pa(i)})$ in (b), or estimates of $\bP(Y_i | S_i)$, $\bP(Y_i | Y_{pa(i)})$ and $\bP(Y_i)$ in (c) for $i=1, \ldots, n$, we can obtain an (unnormalized) estimate of $\bP(Y_1, \ldots, Y_n| S_1, \ldots, S_n)$. Using this, we can estimate $mLPR_i := \bP(Y_i=1| S_1, \ldots, S_n)$ by applying sum-product message-passing \citep{wainwright2008graphical} on $\mG$ to marginalize over $Y_1, \ldots, Y_{i-1}, Y_{i+1}, \ldots, Y_n$.

{\color{black} We estimate $\bP(Y_i | Y_{pa(i)})$, $\bP(Y_i)$, $\bP(S_i | Y_i)$, and $\bP(Y_i | S_i)$ using a training set consisting of $M_{tr}$ objects. Each object is represented by $K$ classifier scores ($\bf S$) and $K$ true class labels ($\bf Y$) corresponding to the $K$ classes. The $M$ objects are assumed to follow an i.i.d. distribution, with $\bf S$ and $\bf Y$ satisfying the probability properties (i), (ii), and (iii) described in Section \ref{sec:notation_model}.
Following previous notations, let $i = (m - 1) \cdot K + k$, where $m$ represents the object index and $k$ the class index. This results in a total of $n_{tr} = K \cdot M_{tr}$ classification cases, with each associated with a true class label and a classifier score. } 

We estimate $\bP(Y_k | Y_{pa(k)})$ as the proportion of positive objects in class $k$ among those whose parent classes are positive.  $\bP(Y_k)$ is calculated as the proportion of objects labeled as positive for class $k$ out of the total number of objects. For $\bP(S_i | Y_i)$, we model $\bP(S_i | Y_i = 0)$ and $\bP(S_i | Y_i = 1)$ as two Gaussian distributions, following the approach of \citet{decoro2007bayesian}. To estimate $\bP(Y_i | S_i)$, we adopt the method proposed by \citet{jiang2014}, which trains a model based on kernel density estimators for $\bP(S_i | Y_i = 0)$ and $\bP(S_i)$. For a new object with known classifier scores but unknown labels, we input the scores into the trained models to compute $\bP(S_i | Y_i)$ and $\bP(Y_i | S_i)$. The estimated mLPRs for the new object are then obtained by applying sum-product message passing to Formula (b) or Formula (c) in \eqref{eq:mLPR}.

{\color{black}The key difference between Formula (c) and Formula (b) lies in the quantity that needs to be estimated: Formula (c) requires estimating $\bP(Y_i | S_i)$, whereas Formula (b) requires estimating $\bP(S_i | Y_i)$. As discussed in detail by Jiang et. al (2014), the estimation of $\bP(Y_i | S_i)$ is more reliable than that of $\bP(S_i | Y_i)$ \citep{jiang2014}. The latter is highly sensitive to both the arrangement of the data and the complexity of its distribution. In particular, if the training samples are concentrated within one or two short intervals and sparse elsewhere, the estimation for $\bP(S_i | Y_i)$ becomes much less reliable. Therefore, throughout the rest of this article, we use Formula (c) to estimate mLPRs rather than Formula (b).
For comparisons of the two approaches, see Supplement D.4. The following theorem provides the convergence rate of the estimation based on Formula (c).}

\begin{theorem}\label{thm:mLPR_rate}
  {\color{black} Under Assumptions A1, A2, A3 provided in Supplement E.1, where A1 indicates densities of $F_{0,k}$ and $F_{1,k}$ are bounded, $k=1, \ldots, K$; A2 provides some mild regularity conditions on the kernel used in the kernel density estimation, and A3 suggests the kernel bandwidth is chosen to be $[(\log M_{tr})/M_{tr}]^{1/3}$,}
  by employing kernel density estimators to estimate $\bP(Y_i | S_i)$ values and employing the empirical estimators to estimate $\bP(Y_i)$ values and $\bP(Y_i|Y_{pa(i)})$ values, it follows that with probability at least $1 - C_1 \cdot K\cdot 2^{d}/{M_{tr}}$
    $$\vert \widehat{mLPR}_i - mLPR_i \vert \leq C_2 \cdot 2^{d} \cdot K\cdot \left (\frac{\log M_{tr}}{M_{tr}}\right )^{\frac{1}{3}}, ~~ i = 1, \ldots, n,$$
    where $d$ is the maximum degree of nodes in $\mG$, and $C_1$ and $C_2$ are positive constants that depend on the densities corresponding to the $F_{1,k}$ and $F_{0,k}$ values.    
  \end{theorem}

Theorem \ref{thm:mLPR_rate} shows that $\widehat{mLPR}$ converges to the true $mLPR$ value at the rate $\calO(M_{tr}^{-1/3}K\cdot 2^{d})$ (ignoring the log factors), where $d$ is the maximum degree of nodes in $\mG$. 
This theorem suggests that accurate estimates of mLPRs can be achieved when there is a sufficient number of training samples (large $M_{tr}$), or when the graph $\mG$ is either sparse (small $d$) or small in size (small $K$). The convergence rate for mLPR is derived from the convergence rate of $\bP(Y_i | S_i)$, which relies on the uniform convergence of the kernel density estimator \citep{jiang2017uniform}. It also depends on the convergence rates of $\bP(Y_i)$ and $\bP(Y_i | Y_{pa(i)})$, which are governed by the Hoeffding bound. 
The proof is given in Supplement E.1.

The procedure described above accounts for the complete dependence between classes, which we refer to as the \textbf{full} version for computing $\widehat{mLPR}$. When the dependency structure is sparse, reasonable approximations can be achieved with reduced computational cost by partially accounting for the dependency structure. Alongside the full version, we introduce the following two approximations.
\begin{itemize}
\item  \textbf{Independence}. Here, we assume that all classes are independent of one another. In this case, {\color{black} $mLPR_i = \bP(Y_i = 1 | \bdvS) = \bP(Y_i = 1 | S_i)$}. This version is referred to as the independence approximation (abbreviated as \textbf{indpt}). It reduces to the estimation of the local precision rate proposed by \citet{jiang2014} for multi-label classification and is utilized in \citet{ho2018hierlpr} for hierarchical multi-label classification.
\item  \textbf{Neighborhood}. In this version, we incorporate local, but not full, dependencies when computing $\widehat{mLPR}$. We assume that Event $i$ is independent of Event $i'$ for $i' \not\in nbh(i) \cup {i}$. Under this assumption, $mLPR_i = \bP(Y_i = 1 | \bdvS_{nbh(i) \cup {i}})$, and we only need to consider $\bP(\bdvY_{nbh(i) \cup {i}} | \bdvS_{nbh(i) \cup {i}})$ when applying Equation \eqref{eq:mLPR}, where $\bdvY_{nbh(i) \cup \{i\}} := \{Y_{i'}| i' \in nbh(i) \text{ or } i'=i\}$ and $\bdvS_{nbh(i) \cup \{i\}} := \{S_{i'}| i' \in nbh(i) \text{ or } i'=i\}$. This version is referred to as the neighborhood approximation (abbreviated as \textbf{nbh}).
\end{itemize}

These three versions offer different advantages depending on the scenario. For instance, the independence or neighborhood approximation is efficient when only weak or local dependencies are observed between classes. Further discussion can be found in Section \ref{sec:synthetic} (where the full version performs best) and Section \ref{sec:disease_example} (where the independence or neighborhood approximation outperforms the full version).

\subsection{ {\bf HierRank}: the Ranking Algorithm based on $\widehat{mLPR}$s}\label{sec:algo}
{\color{black} Given $\widehat{mLPR}$ values, we consider
\begin{equation}
  \widehat{\Objective\/}(\bdpi; \widehat{mLPR}_1, \ldots, \widehat{mLPR}_n) := \sum_{r = 1}^n (n - r + 1) \widehat{mLPR}_{\pi^{-1}(r)},\label{eq:est_cobj}
\end{equation}
which serves as an empirical counterpart to \Objective~\eqref{eq:cobj}. When there is no ambiguity, we simplify the notation and refer to $\widehat{\Objective}(\bdpi; \widehat{mLPR}_1, \ldots, \widehat{mLPR}_n)$ as $\widehat{\Objective}$ in \eqref{eq:est_cobj}.

Since the estimated values $\widehat{mLPR}_i$ may significantly deviate from the true $mLPR_i$, naively sorting them in descending order may not preserve the desirable properties discussed in Section~\ref{sec:mLPR_property}, such as consistency with the given hierarchy. 
{\color{black} 
To address this issue, we introduce \textbf{HierRank}, a sorting algorithm that can operate on any values $V_1, \ldots, V_n$ associated with the $n$ events to be ranked. Specifically, it produces a ranking $\boldsymbol{\pi}$ that is consistent with the class hierarchy and, among all orderings that respect the hierarchy, selects the one that maximizes $\sum_{r=1}^n (n-r+1), V_{\pi^{-1}(r)}$. When $V_i = \widehat{mLPR}_i$, \textbf{HierRank} then maximizes $\widehat{\Objective}(\boldsymbol{\pi};\widehat{mLPR}_1,\ldots,\widehat{mLPR}_n)$ while respecting the class hierarchy.}

Formally, {\bf HierRank}, when applied to $\widehat{mLPR}_1,\ldots,\widehat{mLPR}_n$, solves the following optimization problem, which is the empirical counterpart of \eqref{opt:max_cobj}:
\begin{eqnarray}
    \max_{\bdpi} && \widehat{\Objective\/}(\bdpi;  \widehat{mLPR}_1, \ldots, \widehat{mLPR}_n) \label{opt:max_est_cobj}\\
  \text{s.t.} && \text{$\pmb{\pi}$ is hierarchically consistent.}\nonumber
\end{eqnarray}

{\color{black} 
In other words, \textbf{HierRank} ranks $n$ events distributed across $M$ graphs of identical structure, with each graph (either a tree or a forest comprising $K$ classes) corresponding to one of the $M$ objects, {\color{black}with the following properties:  
(i) It arranges all the $n=M \times K$ events into a single connected chain in which every parent has exactly one child;  
(ii) The derived chain must preserve the original class hierarchy (i.e., if a node is a descendant in the hierarchy, it must also be a descendant in the chain);  
(iii) When applied to $\widehat{mLPR}_1,\ldots,\widehat{mLPR}_n$, it maximizes $\widehat{\Objective}$ among all orderings that respect the class hierarchy.} In more detail, the basic steps of \textbf{HierRank} based on $\widehat{mLPR}_1,\ldots,\widehat{mLPR}_n$} are as follows:
\begin{itemize}
    \item Identify all single-child chains within the $M$ graphs. A single-child chain is a connected subgraph in which each node has at most one child in the original graph.
    \item Merge two single-child chains that share the same root in the original graph into one chain, preserving the class hierarchy. This step is the core of the algorithm and must be guided by the $\widehat{mLPR}$ values to ensure the maximization of $\widehat{\Objective}$ among all orderings that respect the class hierarchy. (In the next two sections, we begin with a toy example to illustrate the {\bf Chain-Merge algorithm} (Algorithm~\ref{algo:chain_merge}), followed by a formal outline of the algorithm. All relevant notation and algorithmic details are provided in Supplement~B.1.
    \item Repeat the above steps until the original tree is fully merged into a single single-child chain.
\end{itemize}

\subsubsection{Demonstration of the Chain-Merge Algorithm: A Toy Example}\label{sec:ToyE}
In Figure~\ref{fig:chain_merge_illustration}, we demonstrate the {\bf Chain-Merge Algorithm} using the subgraph rooted at Node~F from Figure~\ref{fig:model_graph}(b). In this subgraph, we identify two single-child chains: $\{\text{I} \rightarrow \text{J}\}$ and $\{\text{G} \rightarrow \text{H}\}$, both sharing the same root node F in the original graph. The associated $\widehat{mLPR}$ values are: $\widehat{mLPR}_I = 0.3$, $\widehat{mLPR}_J = 0.9$, $\widehat{mLPR}_G = 0.8$, and $\widehat{mLPR}_H = 0.1$.

To merge these two chains, we aim to arrange the four nodes into a single chain while respecting the hierarchy—specifically, node I must precede J, and node G must precede H. Among all valid orderings that satisfy these constraints, we seek the one that maximizes $\widehat{\Objective}$.

According to Formula~(4.2), maximizing $\widehat{\Objective}$ requires ranking nodes with higher $\widehat{mLPR}$ values as early as possible. Ignoring hierarchy constraints, the optimal ranking would be $\{\text{J}(0.9) \rightarrow \text{G}(0.8) \rightarrow \text{I}(0.3) \rightarrow \text{H}(0.1)\}$, yielding $\widehat{\Objective} = 4 \times 0.9 + 3 \times 0.8 + 2 \times 0.3 + 1 \times 0.1 = 6.7$. However, this ordering violates the hierarchy by placing J before I.

\begin{figure}[tp]
  \centering
   \includegraphics[width=\linewidth]{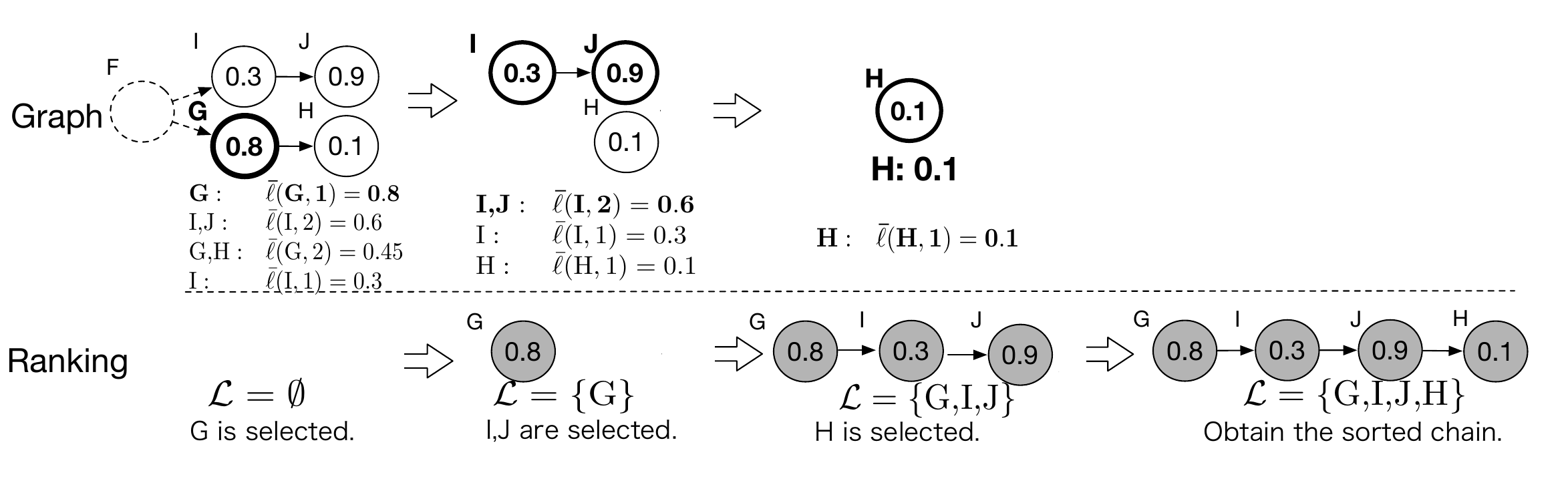}
   \caption{Example of the merging process in Algorithm 1: Bold circles indicate the sub-chain with the highest {\color{black} average $\widehat{mLPR}$ values}, while gray-filled circles represent the ranking produced by the merging procedure.} 
     \label{fig:chain_merge_illustration}
\end{figure}

{\color{black}Using $\widehat{mLPR}_1,\ldots,\widehat{mLPR}_n$}, the {\bf Chain-Merge Algorithm} aims to find the ordering that maximizes $\widehat{\Objective}$ among all those that respect the given hierarchy. We now describe and justify each step of the algorithm as applied to this toy example:

\begin{itemize}
\item {\bf Candidate nodes or sub-chains to rank first.} We define a sub-chain as any contiguous prefix of a chain, starting from its root and maintaining the original ordering. For the four nodes under consideration, the algorithm identifies the following candidate nodes/sub-chains that may be ranked first: $\{\text{I}\}, \{\text{G}\}, \{\text{I} \rightarrow \text{J}\}$, and $\{\text{G} \rightarrow \text{H}\}$. The nodes $\{\text{J}\}$ and $\{\text{H}\}$ are not eligible to be ranked first due to hierarchy constraints.

\item {\bf Determine which node or subchain should appear first in the ordering.} Comparing $\{\text{I}\}, \{\text{G}\}$, and $\{\text{G} \rightarrow \text{H}\}$, it is clear that $\{\text{G}\}$ should be preferred as it has a higher $\widehat{mLPR}$ value than both node $\text{I}$ and node $\text{H}$. To choose between $\{\text{G}\}$ and $\{\text{I} \rightarrow \text{J}\}$, it is equivalent to compare $\{\text{I}\rightarrow \text{J} \rightarrow \text{G} \}$ and $\{\text{G}\rightarrow \text{I} \rightarrow \text{J} \}$. Compared to $\{\text{G}\rightarrow \text{I} \rightarrow \text{J} \}$, the $\widehat{\Objective\/}$ {\color{black} value} for $\{\text{I}\rightarrow \text{J} \rightarrow \text{G} \}$ is higher by $(\widehat{mLPR}_I + \widehat{mLPR}_J)$ (due to both nodes $\text{I}$ and $\text{J}$ being ranked one position higher), but lower by $2 \times \widehat{mLPR}_G$ (due to node $\text{G}$ being ranked two positions lower). Therefore, we compare the average of $\widehat{mLPR}_I$ and  $\widehat{mLPR}_J$, which is $(0.3+0.9)/2 = 0.6$, with $\widehat{mLPR}_G = 0.8$. Since $0.8 > 0.6$, $\{\text{G} \rightarrow \text{I} \rightarrow \text{J}\}$ is preferred, and $\{\text{G}\}$ should be ranked first.

{\bf Remark}: This step highlights the need to compare average $\widehat{mLPR}$ values of sub-chains with individual nodes. This is necessary when the $\widehat{mLPR}$ values of some parent nodes are lower than those of their descendants in the sub-chain, potentially resulting in a higher $\widehat{\Objective}$ if the nodes within the sub-chain are kept together. We calculate the average because moving a sub-chain of $l$ nodes up by one position results in pushing some other node down by $l$ positions. Therefore, to determine the optimal ranking, we compare the sub-chain’s average $\widehat{mLPR}$ against the competing node’s value.

\item {\bf Determine the next node or sub-chain to rank.} After ranking $\{\text{G}\}$ first, the remaining nodes are $\{\text{I}\}, \{\text{J}\}$, and $\{\text{H}\}$. The next candidates are $\{\text{I}\}, \{\text{I} \rightarrow \text{J}\}$, and $\{\text{H}\}$. Since $\widehat{mLPR}_H < \widehat{mLPR}_I < (\widehat{mLPR}_I + \widehat{mLPR}_J)/2$, the optimal choice is to place $\{\text{I} \rightarrow \text{J}\}$ after $\{\text{G}\}$.

{\bf Remark:} The choice of $\{\text{I} \rightarrow \text{J}\}$ ensures that J follows I immediately while the choice of $\{\text{I}\}$ allows an insertion of another node (like H) between I and J. Since $\widehat{mLPR}_H < \widehat{mLPR}_I < (\widehat{mLPR}_I +  \widehat{mLPR}_J)/2$, we conclude that $\{\text{I} \rightarrow \text{J}\}$ should immediately follow $\{\text{G}\}$ to ensure an optimal $\widehat{\Objective\/}$ value. 

\item {\bf Final ranking.} With $\{\text{G} \rightarrow \text{I} \rightarrow \text{J}\}$ formed, we append the final node $\{\text{H}\}$ to obtain the final merged chain: $\{\text{G} \rightarrow \text{I} \rightarrow \text{J} \rightarrow \text{H}\}$.

\end{itemize}

The ranking produced by the above approach satisfies hierarchical consistency, as it preserves the relative ordering of nodes within each chain during the merging process. Moreover, this ranking maximizes $\widehat{\Objective\/}$ among all possible orderings of the four nodes that respect the hierarchy, as it effectively sorts the $\widehat{mLPR}$ values of individual nodes or the average $\widehat{mLPR}$ values of sub-chains in descending order. As discussed earlier, we consider sub-chains because, in some cases, the nodes within a sub-chain need to be kept together. For an intuitive interpretation, a sub-chain can be treated as a pseudo-node, with its associated $\widehat{mLPR}$ {\color{black} value} defined as the average of the $\widehat{mLPR}$ values of all nodes it contains. It allows us to account for tradeoffs in position shifts during ranking (moving a sub-chain of length $l$ up by one position pushes another single node down by $l$ positions). As a result, when comparing a sub-chain to a single node to determine which should be ranked first, we effectively compare the sub-chain’s average $\widehat{mLPR}$ to the single node’s $\widehat{mLPR}$ value.}

}

\subsubsection{The Chain-Merge Algorithm: A Formal Outline}\label{sec:AlgFormalSketch}

{\color{black} 
}
We formally present the \textbf{Chain-Merge Algorithm} in Algorithm~\ref{algo:chain_merge}, {\color{black}which is based on values $V_1,\ldots,V_n$ associated with the events.} 

\begin{algorithm}[tp]
\raggedright
\caption{Chain-Merge algorithm.}\label{algo:chain_merge} 
\textbf{Notation: $\mC_r(h)$ is a chain of length $h$ whose root is $r$ and whose node has at most one child; the length $h$ can be abbreviated if it refers to the entire chain. {\color{black} $\bar\ell_{\bdvV}(r, h)$ is the average of $\{V_i: i\in \mC_r(h)\}$, where $V_i$ is the value associated with node $i$ (e.g., $V_i=\widehat{mLPR}_i$).}\\
}
\textbf{Input: }$p$ chains $\mD = \{\text{node} \in \mC_r: r = r_1, \ldots, r_p\}$, node {\color{black} values $\bdvV$}.\\
\textbf{Procedure: }
\begin{algorithmic}[1]
\STATE Set $\mL=\emptyset$.
\STATE Compute {\color{black} $\{\bar\ell_{\bdvV}(r, h): h = 1, \dots, |\mC_r|, ~~ r = r_1, \ldots, r_p\}$}.
\WHILE{$\mD \neq \emptyset$}
\STATE {\color{black} $(r',h') = \arg \underset{C_r(h) \subset \mD}{\max} \bar\ell_{\bdvV}(r,h).$}
\STATE $\mL \gets \mL \oplus C_{r'}(h')$, where '$\oplus$' indicates the concatenation of two sequences.\\
\STATE $\mD \gets (\mD \backslash C_{r'}) \cup (C_{r'}\backslash C_{r'}(h'))${\color{black}, where '$\backslash$' is the set difference operation.}\\
\STATE Update the mean {\color{black} values} of the remaining nodes as in Step 2. 
\ENDWHILE      
\end{algorithmic}
\textbf{Output: }$\mL$.
\end{algorithm}

{\color{black} 
A step-by-step illustration of Algorithm~\ref{algo:chain_merge}, {\color{black} using the toy example in Figure~\ref{fig:chain_merge_illustration} again (where $V_i = \widehat{mLPR}_i$), is shown below:} 
\begin{itemize} \item \textbf{Initialization:} Start with an empty ranked list, $\mL = \emptyset$.

\item \textbf{Step 1:}
With the common root $\text{F}$, consider four sub-chains: $\{\text{G}\}$, $\{\text{G} \rightarrow \text{H}\}$, $\{\text{I}\}$, and $\{\text{I} \rightarrow \text{J}\}$, with mean {\color{black} values} of $0.8$, $0.45$, $0.3$, and $0.6$, respectively. The sub-chain $\{\text{G}\}$ is selected as it has the highest mean {\color{black} value}.
Remove the sub-chain $\{\text{G}\}$ from the original graph and add it to $\mL$. Update $\mL$ as $\mL = \{\text{G}\}$.

\item \textbf{Step 2:}
In the remaining graph, three sub-chains remain: $\{\text{H}\}$, $\{\text{I}\}$, and $\{\text{I} \rightarrow \text{J}\}$, with updated mean {\color{black} values} of $0.1$, $0.3$, and $0.6$, respectively. Among these, the sub-chain $\{\text{I} \rightarrow \text{J}\}$ is selected due to its highest mean {\color{black} value}.
Remove the sub-chain $\{\text{I} \rightarrow \text{J}\}$ from the remaining graph and add it to $\mL$. Update $\mL$ as $\mL = \{\text{G}, \text{I}, \text{J}\}$.

\item \textbf{Step 3:}
The single node $\{\text{H}\}$ is selected as it is the last remaining node.
Remove node $\{\text{H}\}$ from the graph and add it to $\mL$. Since no nodes remain in the graph, the final ranked list is $\mL = \{\text{G}, \text{I}, \text{J}, \text{H}\}$. \end{itemize}}

The ranking produced by the Chain-Merge algorithm satisfies hierarchical consistency because it preserves the relative ordering of nodes within each chain during the merging process. Moreover,  {\color{black} when applied to $\widehat{mLPR}$ values,} this ranking maximizes $\widehat{\Objective}$ for the subgraph of interest among all possible topological orderings, as the algorithm effectively sorts subchains based on their {\color{black} average $\widehat{mLPR}$ values} in descending order. Intuitive reasoning supporting this approach was provided in the previous section.

\begin{algorithm}[tp]
\raggedright  
\caption{HierRank algorithm (sketch) for ranking the nodes in the tree hierarchy.}\label{algo:HierRank_tree_sketch}
\textbf{Input: }{A tree graph $\mG$, node {\color{black} values $\bdvV$} (e.g., $\widehat{mLPR}$ values).}\\
\textbf{Procedure: }
\begin{algorithmic}[1]
\WHILE{There exists a node that has more than one child node.}
\STATE  Identify a node $v$ such that i) it has at least two child nodes, ii) each of its child/descendant nodes has at most one child node in $\mG$.
\STATE  Merge all the chains that starts from $v$ into one chain by the Chain-Merge algorithm.
\ENDWHILE
\IF{there remain multiple chains}
\STATE Apply the Chain-Merge algorithm to merge these chains into one chain.
\ENDIF
\STATE Let $\mL$ be the resulting chain.
\end{algorithmic}
\textbf{Output: }{a ranking $\mL$}.
\end{algorithm}
{\color{black}
\subsubsection{The HierRank Algorithm} \label{HierRank}}
Building on the Chain-Merge algorithm, HierRank resolves \eqref{opt:max_est_cobj} by iteratively merging chains stemming from the same node into a single chain. 

A sketch of {\bf HierRank} is provided in Algorithm \ref{algo:HierRank_tree_sketch}, while the formal version of this procedure is detailed in Algorithm $2^{\prime}$ in Supplement B.1. Since {\bf HierRank} is based on the Chain-Merge algorithm, it inherits the desired optimality property; specifically, {\color{black}when applied to $\widehat{mLPR}$ values,} it produces a topological ordering of $\mG$ that achieves the maximum value of $\widehat{\Objective}$ among all possible topological orderings. This claim is formally presented in Theorem \ref{thm:opt_HierRank_tree}, with the proof presented in Supplement E.2.

\begin{theorem}\label{thm:opt_HierRank_tree}
{\color{black}When applied to $\widehat{mLPR}$ values,} {\bf HierRank}, {\color{black} which merges all of the events distributed across one or multiple disjoint trees into a single chain, achieves }an optimal topological ordering w.r.t \eqref{opt:max_est_cobj}.
\end{theorem}

We have two remarks regarding {\bf HierRank}.
First, the time complexity of {\bf HierRank} is $\calO(K^3)$ for each object, resulting in a computational cost of $\calO(MK^3 + MK \log M)$ for ranking $K$ nodes/classes across $M$ objects. A faster version of the algorithm can be derived by eliminating exhaustive merging and repeated computations at each iteration, reducing the computational cost to $\calO(n \log n)$, where $n = MK$. Details of this faster version, Algorithm $2^{\prime\prime}$, are provided in Supplement B.3. Second, the current version of {\bf HierRank} is designed for tree graphs. It can, however, be extended to work with directed acyclic graphs (DAGs). The details of this extension are given in Supplement B.4.


\subsection{A Unified procedure of mLPR-based Decision-Making in HMC}\label{sec:cutoff_selection}
In this section, we present a unified procedure that integrates mLPR estimation with a ranking algorithm and provides a method for determining a cutoff on the ranked list to make a final decision. The inputs include a training dataset containing class-wise classifier scores and true labels, along with a graph organizing the classes. The outputs are the trained models for mLPR estimation and the suggested cutoff for the final classification decision. The procedure is as follows:
\begin{enumerate} \item[0.] Split the original training dataset into two subsets: a training subset and a validation subset.
\item Train the models on the training subset to obtain the parameters required for $\widehat{mLPR}$ computation, as described in Section \ref{sec:mLPR_estimation}.
\item Calculate $\widehat{mLPR}s$ for the events in the validation subset. 
\item Use {\bf HierRank} to rank the events in the validation subset based on $\widehat{mLPR}$ values, resulting in a single ranking. 
\item Using the ranking from the previous step, compute the false discovery proportion (FDP) and the $F_1$ score when the top {\color{black} $\kappa \times 100\%$ of events are identified as positives, where $0 < \kappa < 1$. FDP is defined as the ratio of false predicted positives to the total number of predicted positives, while the $F_1$ score is the harmonic mean of the recall and precision rates. Select the optimal $\hat{\kappa}$ that either maximizes the $F_1$ score or targets a specific FDR value. Alternative selection criteria can be applied in a similar manner to determine the optimal $\hat{\kappa}$. {\color{black} 
When the number of objects or classes is small, or equivalently, when the number of events to be considered is small, the decision threshold may also be set using a $\widehat{mLPR}$ (or average $\widehat{mLPR}$) cutoff calibrated by the obtained optimal $\hat{\kappa}$. 
}}
\end{enumerate} 

In the procedure outlined above, the data-splitting process in Step 0 requires careful consideration to ensure that the distributions of the training and validation sets are similar and representative. When the input dataset is sufficiently large, random splitting is adequate. 

{\color{black} 
For the testing objects with classifier scores but unknown labels, the trained models from Step 1 are used to compute their $\widehat{mLPR}$ values across the classes. These values are then used by {\bf HierRank} to rank the classification events for these objects. Finally, the top $\hat{\kappa} \times 100\%$ of ranked events, where $\hat{\kappa}$ is determined in Step 4, are labeled as positive, and the rest as negative. Alternatively, the top events exceeding a $\widehat{mLPR}$ (or average $\widehat{mLPR}$) cutoff, calibrated based on $\hat{\kappa}$, may be labeled as positive.}

We evaluated the performance of {\bf HierRank} and the unified procedure described above using synthetic and real datasets, as discussed in Section \ref{sec:evaluation}.

\section{Experiments}\label{sec:evaluation}
We compared our method with off-the-shelf HMC approaches on one synthetic dataset and two real datasets, evaluated based on two criteria as described below.

\subsection{Setup}\label{sec:eval_setup}
\noindent \textit{Benchmarked Methods}. We compared our method against several methods listed below. Details of these methods are provided in Table S1 in Supplement D.3.

\begin{itemize}
\item 
{\bf Raw score-based methods.} Specifically, we consider two approaches: \textbf{Raw-NaiveSort} and \textbf{Raw-HierRank}, which apply NaiveSort (introduced in Section \ref{sec:mLPR_sorting}) and HierRank, respectively, to rank the events based on raw classifier scores. {\color{black} Note that ranking the raw classifier scores directly is reasonable when higher classifier scores indicate a higher likelihood of positive class labels, which is typically expected for a well-performing classifier.}

\item \textbf{Clus-HMC variants}. {\color{black} Clus-HMC is a state-of-the-art, non-neural-network-based, end-to-end HMC method that performs classification while simultaneously addressing hierarchical dependencies \citep{blockeel2002, blockeel2006decision, vens2008, schietgat2010predicting}. We implemented Clus-HMC in the \textit{R} language independently, providing two versions: one with bagging, referred to as \textbf{ClusHMC-bagging}, and another without ensembling, referred to as \textbf{ClusHMC-vanilla}. We validated the correctness of our implementation by comparing its outputs against those produced by the \textit{CLUS} software\footnote{https://dtai.cs.kuleuven.be/software/clus/}}.  
\item \textbf{HIROM variants}. HIROM is a state-of-the-art two-stage HMC classifier \citep{bi2015bayes} and produces Bayes-optimal predictions that minimize a series of hierarchical risks, such as risks based on hierarchical ranking loss and hierarchical Hamming loss. We implemented two variants of HIROM in the \textit{R} language, referred to as \textbf{HIROM-hier.ranking} and \textbf{HIROM-hier.Hamming}, respectively.
\item \textbf{C-HMCNN}. 
Coherent Hierarchical Multi-Label Classification Networks (C-HMCNN) is a state-of-the-art neural network for HMC \citep{giunchiglia2020neurips, giunchiglia2021multi}, designed to leverage hierarchical information to produce predictions that adhere to constraints and improve performance. We used the publicly available code\footnote{https://github.com/EGiunchiglia/C-HMCNN} provided by the authors. 
\item \pmb{$\widehat{mLPR}$}\textbf{-based methods}. 
Our method, \pmb{$\widehat{mLPR}$}\textbf{-HierRank}, estimates mLPRs and ranks events based on $\widehat{mLPR}$ using HierRank. It was implemented in an R package\footnote{https://github.com/Elric2718/mLPR}. Alternatively, events can be ranked based on $\widehat{mLPR}$ using NaiveSort; this method is referred to as \pmb{$\widehat{mLPR}$}\textbf{-NaiveSort}. To distinguish among the three approaches for estimating mLPR described in Section~\ref{sec:mLPR_estimation}—namely, the independence approximation, neighborhood approximation, and the full version—we append the postfixes “-indpt,” “-nbh,” and “-full” to each method name (e.g., \Method\/HierRank-full). 

\end{itemize}

\noindent \textit{Evaluation Criteria}. We used the following two criteria to evaluate the ranking results produced by the methods mentioned above.

\begin{itemize}
\item {\color{black} \textbf{Evaluation Criterion I: Precision--Recall (PR) at different cutoffs $\kappa$.} We assessed the rankings by calculating the recall (i.e., \# true predicted positives / \# all true positives) and precision (i.e., \# true predicted positives / \# predicted positives) values by taking the top $\kappa \times 100\%$ of events as positive and the others as negative, with $\kappa = 0.05, 0.1, 0.2, 0.3 \text{ and } 0.5$.}

\item \textbf{Evaluation Criterion II: False discovery proportion (FDP) and $F_1$ score under proposed cutoff selection}. {\color{black}We used the procedure introduced in Section \ref{sec:cutoff_selection} to determine the cutoff $\hat{\kappa}$. The $F_1$ score was computed on the test dataset when $\hat{\kappa}$ was selected by maximizing the $F_1$ score on the validation dataset. Similarly, FDP was computed on the test dataset when $\hat{\kappa}$ was selected to target an FDR of $\alpha \times 100\%$, with $\alpha = $ 0.01, 0.05, 0.1, and 0.2.}
\end{itemize}

\noindent \textit{{\color{black}Evaluation Datasets}}. Each HMC method was evaluated using three datasets: 1) a synthetic dataset, 2) the disease diagnosis dataset from \citet{huang2010}, and 3) the RCV1v2 dataset \citep{lewis2004rcv1}. Details of these datasets are provided below.

\begin{figure}[bp]
  \centering
 \includegraphics[width=0.5\linewidth]{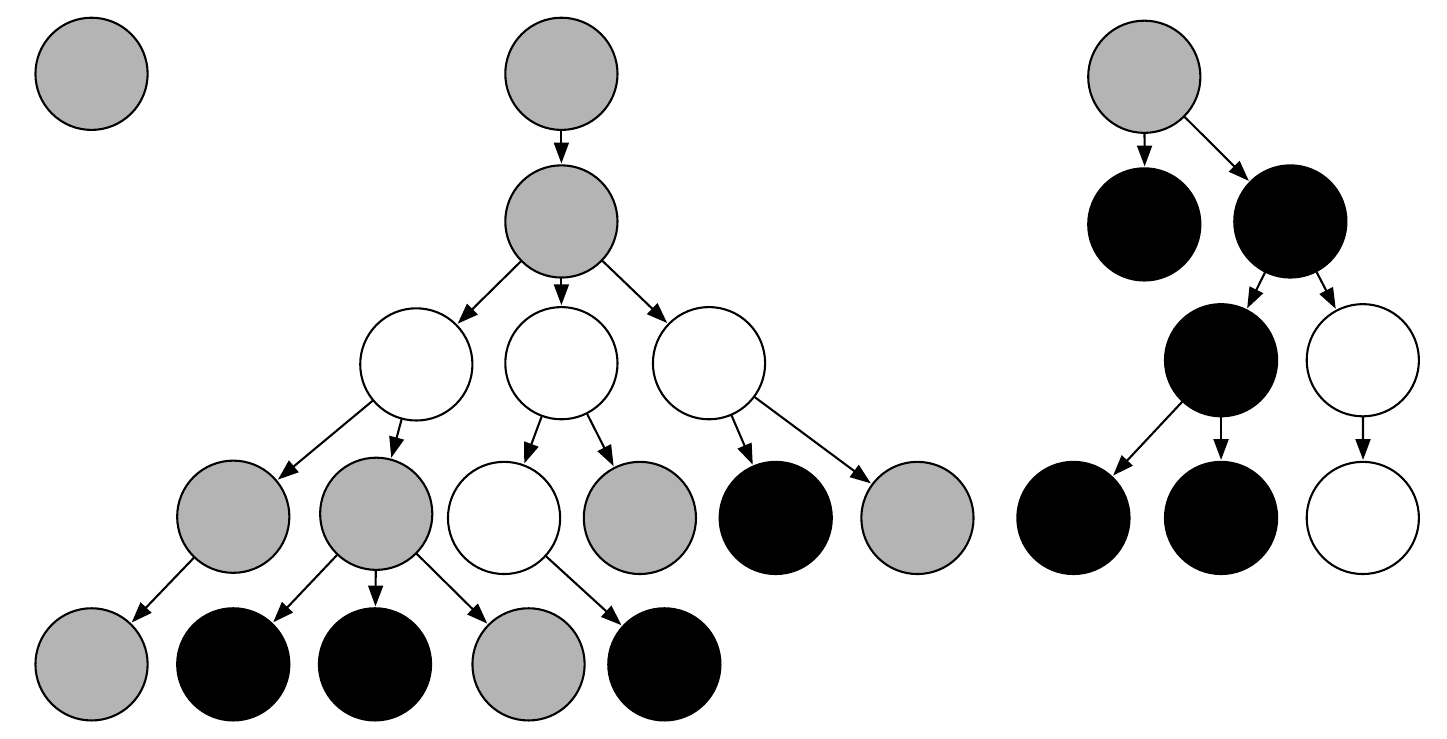}
  \caption{Class trees of the synthetic dataset: White, gray, and black indicate classes with high, medium, and low classifier quality, respectively.}\label{class_sim_settings}
\end{figure}

\begin{itemize}
\item \textbf{Synthetic dataset}. The simulated dataset comprised 25 classes, with a hierarchy shown in Figure \ref{class_sim_settings}.
There were 50,000 training objects and 10,000 testing objects. For each object, we generated the true class labels as follows: The positive case probability at root $\bP(Y_{\text{root}} = 1)$ and the conditional probabilities $\bP(Y_i = 1 | Y_{pa(i)} = 1)$ were randomly generated from a uniform distribution, subject to the constraint that there must be a minimum of 15 positive instances of any class in the training set (i.e., a minimum prevalence of $0.3\%$). Given the true class status, the score $S$ was generated from distributions specific to each class and status: data were generated from a Beta($\eta$, 3.5) distribution for the negative case and a Beta(3.5, $\eta$) distribution for the positive case. Here, we set $\eta$ to 2, 5.5, or 4 in descending order of the absolute mean difference between the two Beta distributions (i.e., $\vert 3.5-\eta\vert/\vert 3.5+\eta\vert$), which corresponds to the high, medium, or low classifier quality, respectively. The quality of classifier refers to the difficulty in distinguishing between the positives and the negatives. Specifics of the score generation mechanism are summarized in Table \ref{tbl:quality_nodes}.

\begin{table}[tp]
\caption{Score distributions by classifier quality for the synthetic dataset.}     
  \centering
  \begin{tabular}{l|cc|cc}
    \hline
    Quality & Positive instance & Negative instance & Absolute mean difference & Node color\\
    \hline
    High    & Beta$(3.5, 2)$      &  Beta$(2, 3.5)$ &  $3/11$    & white\\
    Medium  & Beta$(3.5, 5.5)$      &  Beta$(5.5, 3.5)$ & $2/9$     & gray\\
    Low     & Beta$(3.5, 4)$    &  Beta$(4, 3.5)$ & $1/15$   & black\\
    \hline
  \end{tabular}
\label{tbl:quality_nodes}  
\end{table}
\end{itemize}

\begin{figure}[bp]
\hspace*{-0.4in}     
\includegraphics[width=5.2in, height=3in]{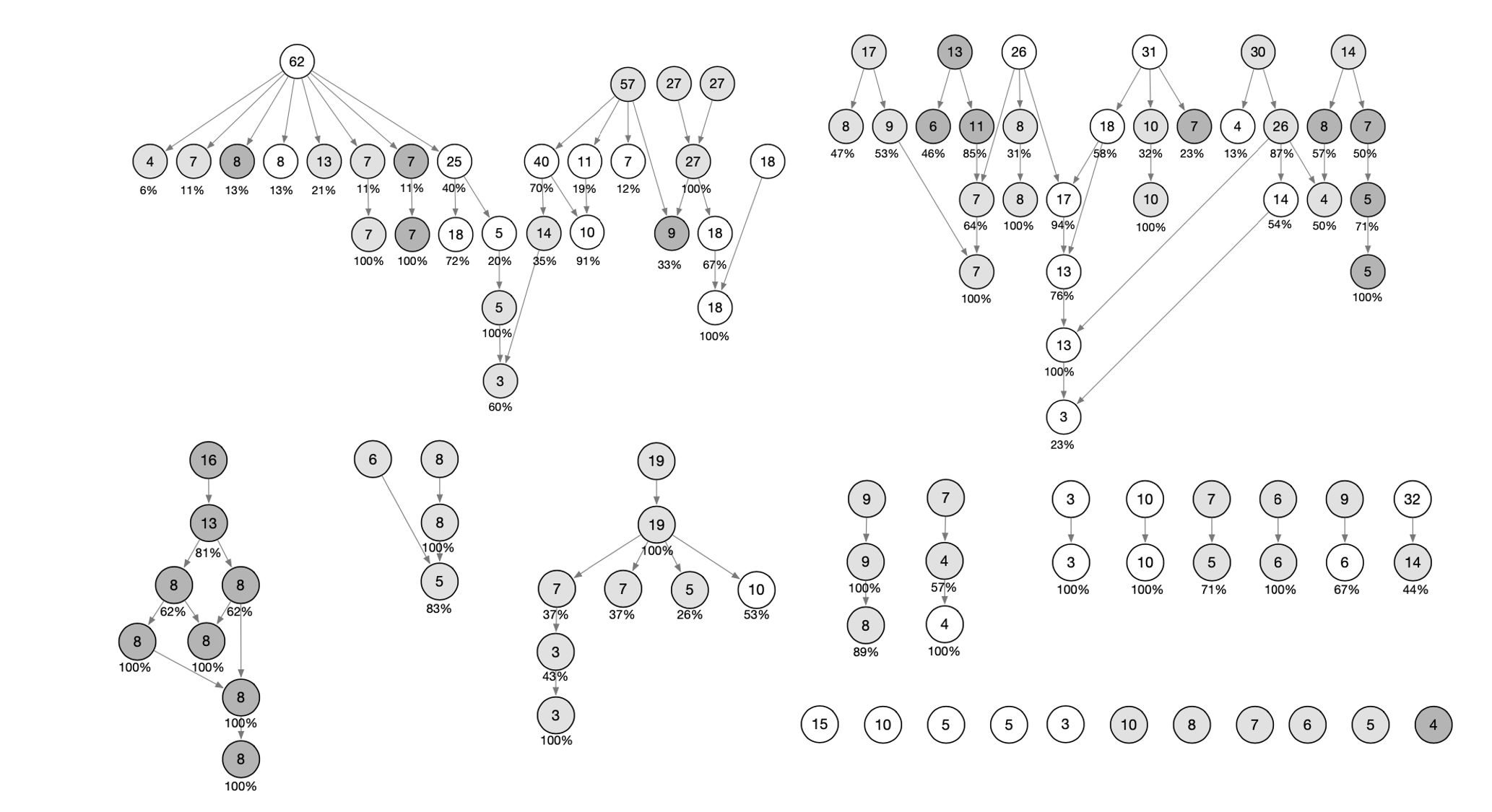}
\caption{Structure of the classes of the disease diagnosis dataset. The grayscale corresponds to the node quality: white indicates that a node's base classifier has an area under the curve (AUC) value of receiver operating characteristic curve (ROC) within the interval $(0.9, 1]$; light gray, $(0.7, 0.9]$; and dark gray, $(0, 0.7]$. The values inside the circles indicate the number of positive cases, and the values underneath indicate the maximum percentage of positive cases shared with a parent node.}\label{fig:disease_diag_graph_all}
\end{figure}

\begin{itemize}
\item \textbf{Disease diagnosis dataset}. \citet{huang2010} studied the problem of disease diagnosis using the UMLS DAG and  public microarray datasets from the National Center for Biotechnology Information (NCBI) Gene Expression Omnibus (GEO). They collected data from 100 studies, including a total of {\color{black} 196 datasets} and 110 disease classes. The 110 classes are represented by 110 nodes, which are grouped into 24 connected DAGs (Figure \ref{fig:disease_diag_graph_all}); see Supplement D.2 for further detail. In general, this graph has three properties: (i) It is shallow rather than deep; (ii) It is scattered rather than highly connected; (iii) Data redundancy occurs, e.g., for some nodes, all of the positive instances are also positive for the associated child nodes.

\item \textbf{RCV1v2}. {\color{black} This is the Reuters Corpus Volume I (RCV1) dataset from a text categorization application.} The raw data consists of $800,000$ manually categorized newswire stories. For this study, we used the corrected version, RCV1v2, which includes $30,000$ stories and $103$ categories. These categories are organized into four hierarchical trees representing different groups: \textit{corporate/industrial} with $34$ classes, \textit{economics} with $26$ classes, \textit{government/social} with $11$ classes, and \textit{markets} with $32$ classes. The class hierarchy and additional details can be found in \citet{lewis2004rcv1}.
\end{itemize}

{\color{black} 
These three datasets each represent a distinct data quality scenario that aligns with one of the three variants of our method: $\widehat{mLPR}$-HierRank-indpt, $\widehat{mLPR}$-HierRank-full, and $\widehat{mLPR}$-HierRank-nbh, respectively. The synthetic dataset, with the highest-quality training samples, offers the most favorable conditions for estimating $\widehat{mLPR}$ accurately. In contrast, the disease diagnosis dataset, with limited and noisy training data, poses the greatest challenge.\footnote{\color{black} Accordingly, we report C-HMCNN results only on the disease diagnosis dataset, as it is the most challenging. Its performance on the other datasets is comparable to that of our $\widehat{mLPR}$-based methods and is omitted for brevity.} The RCV1v2 dataset lies between the two, moderately difficult due to its scale and class diversity. RCV1v2 is also used to illustrate the importance of training the first-stage classifiers and mLPR models on separate data subsets.
}

\subsection{Results on Synthetic Dataset}\label{sec:synthetic}
{\color{black}For end-to-end methods, the first-stage classifier scores were used as the input in the form of a one-dimensional feature. 
We present below the results of the benchmarked methods. For our HierRank-based method, we report only the results using the full-version approach for estimating $\widehat{mLPR}$. Additional results are provided in Supplement D.4.}\\

\noindent \textit{Results Based on Evaluation Criterion I}. For each two-stage method, we followed the procedure in Section \ref{sec:cutoff_selection} to obtain the ranking of the events in the testing dataset. We considered the top $\kappa \times 100\%$ of events as predicted positive events and calculated the corresponding precision and recall for $\kappa = $ 0.05, 0.1, 0.2, 0.3, and 0.5. The results, shown in Table \ref{tbl:synthetic_result}, led to the following observations. 
\begin{enumerate}

\item The performance of \Method\/NaiveSort was similar to that of \Method\/HierRank, and both outperformed the other methods. This may be because fully incorporating the hierarchical information (i.e., using the full version for estimating mLPRs) and having a large training sample size result in $\widehat{mLPR}$ values that closely approximate the true mLPRs. In this scenario, NaiveSort and HierRank produce similar rankings, as suggested by Proposition \ref{prop:hierarchy_consistency}, and these rankings are expected to be nearly optimal with respect to \Objective, as discussed in Section \ref{sec:mLPR_property}


\item Raw-NaiveSort and Raw-HierRank performed poorly compared to the other methods. {\color{black} This poor performance is attributed to the distributional differences between the raw scores for different classes. Despite their poor performance, Raw-HierRank significantly outperformed Raw-NaiveSort.} This difference arises because NaiveSort, which directly sorts raw scores, can violate the hierarchy, whereas HierRank ensures a ranking that respects the hierarchy. Overall, HierRank consistently outperforms NaiveSort.
\end{enumerate}

\begin{table}[bp]
  \caption{Recall and precision (prec) values on the synthetic testing dataset. Here, $\kappa := $\# predicted positives / \# all events. The highest values in each column are shown in bold. All values are percentages.}\label{tbl:synthetic_result}  
    \centering  
\begin{tabular}{l|p{0.45cm}p{0.45cm}|p{0.45cm}p{0.45cm}|p{0.45cm}p{0.45cm}|p{0.45cm}p{0.45cm}|p{0.45cm}p{0.45cm}}
\hline
\multicolumn{1}{r|}{$\kappa~$}                         & \multicolumn{2}{p{0.45cm}|}{0.05}                          & \multicolumn{2}{p{0.45cm}|}{0.1}                           & \multicolumn{2}{p{0.45cm}|}{0.2}                           & \multicolumn{2}{p{0.45cm}|}{0.3}                           & \multicolumn{2}{p{0.45cm}}{0.5}                           \\ \hline
 Method                                    & \multicolumn{1}{p{0.45cm}}{recall}        & prec          & \multicolumn{1}{p{0.45cm}}{recall}        & prec          & \multicolumn{1}{p{0.45cm}}{recall}        & prec          & \multicolumn{1}{p{0.45cm}}{recall}        & prec          & \multicolumn{1}{p{0.45cm}}{recall}        & prec          \\ \hline\hline
Raw-NaiveSort                               & \multicolumn{1}{p{0.45cm}}{5.3}           & 13.9          & \multicolumn{1}{p{0.45cm}}{8.5}           & 11.2          & \multicolumn{1}{p{0.45cm}}{14.0}          & 19.2          & \multicolumn{1}{p{0.45cm}}{20.2}          & 8.9           & \multicolumn{1}{p{0.45cm}}{35.4}          & 9.3           \\ 
Raw-HierRank                                & \multicolumn{1}{p{0.45cm}}{5.1}           & 13.5          & \multicolumn{1}{p{0.45cm}}{13.5}          & 17.8          & \multicolumn{1}{p{0.45cm}}{30.4}          & 20.0          & \multicolumn{1}{p{0.45cm}}{45.5}          & 20.0          & \multicolumn{1}{p{0.45cm}}{69.1}          & 18.2          \\ \hline
ClusHMC-vanilla                             & \multicolumn{1}{p{0.45cm}}{32.7}          & 86.2          & \multicolumn{1}{p{0.45cm}}{54.4}          & 73.0          & \multicolumn{1}{p{0.45cm}}{76.6}          & 50.5          & \multicolumn{1}{p{0.45cm}}{85.8}          & 37.7          & \multicolumn{1}{p{0.45cm}}{93.9}          & 24.8          \\ 
ClusHMC-bagging                             & \multicolumn{1}{p{0.45cm}}{33.9}          & 89.3          & \multicolumn{1}{p{0.45cm}}{56.7}          & 74.7          & \multicolumn{1}{p{0.45cm}}{76.8}          & 50.6          & \multicolumn{1}{p{0.45cm}}{86.5}          & 38.0          & \multicolumn{1}{p{0.45cm}}{94.3}          & 24.9          \\ \hline
HIROM-hier.ranking                          & \multicolumn{1}{p{0.45cm}}{35.5}          & 93.5          & \multicolumn{1}{p{0.45cm}}{55.6}          & 81.0          & \multicolumn{1}{p{0.45cm}}{81.1}          & 53.5          & \multicolumn{1}{p{0.45cm}}{84.1}          & 36.9          & \multicolumn{1}{p{0.45cm}}{88.7}          & 23.4          \\ 
HIROM-hier.Hamming                          & \multicolumn{1}{p{0.45cm}}{35.7}          & 94.2          & \multicolumn{1}{p{0.45cm}}{59.7}          & 78.8          & \multicolumn{1}{p{0.45cm}}{85.4}          & 56.3          & \multicolumn{1}{p{0.45cm}}{89.6}          & 39.4          & \multicolumn{1}{p{0.45cm}}{92.6}          & 24.4          \\ \hline
\Method\/NaiveSort-full & \multicolumn{1}{p{0.45cm}}{\textbf{36.6}} & \textbf{96.6} & \multicolumn{1}{p{0.45cm}}{\textbf{64.7}} & \textbf{85.3} & \multicolumn{1}{p{0.45cm}}{\textbf{86.8}} & \textbf{57.2} & \multicolumn{1}{p{0.45cm}}{\textbf{93.9}} & \textbf{41.3} & \multicolumn{1}{p{0.45cm}}{98.6} & \textbf{26.0} \\ 
\Method\/HierRank-full & \multicolumn{1}{p{0.45cm}}{\textbf{36.6}} & \textbf{96.6} & \multicolumn{1}{p{0.45cm}}{\textbf{64.7}} & \textbf{85.3} & \multicolumn{1}{p{0.45cm}}{\textbf{86.8}} & \textbf{57.2} & \multicolumn{1}{p{0.45cm}}{\textbf{93.9}} & \textbf{41.3} & \multicolumn{1}{p{0.45cm}}{\textbf{98.7}} & \textbf{26.0} \\ \hline
\end{tabular}
\end{table}


\noindent \textit{Results Based on Evaluation Criterion II}. 
We split the original testing dataset equally into a validation set and a new testing dataset ($5,000$ objects in each) and followed the procedure outlined in Section \ref{sec:cutoff_selection} to determine the cutoff $\hat{\kappa}$ based on the validation dataset.
{\color{black} 
The FDP values were computed on the new testing dataset when $\hat{\kappa}$ was chosen to target an FDR of $\alpha \times 100\%$, with $\alpha = $0.01, 0.05, 0.1, and 0.2. Similarly, The $F_1$ scores were calculated on the new testing dataset when $\hat{\kappa}$ was selected by maximizing the $F_1$ score on the validation dataset.
In addition to the FDP and $F_1$ scores, we also reported the corresponding discovery proportion (d.p. := \# predicted positives / \# all events).} The results are summarized in Tables \ref{tbl:synthetic_cutoff_FDR} and \ref{tbl:synthetic_cutoff_fmax}. 

As shown in Table \ref{tbl:synthetic_cutoff_FDR}, \Method\/HierRank-full made the most discoveries while effectively controlling the FDP. Table \ref{tbl:synthetic_cutoff_fmax} demonstrates that \Method\/HierRank-full achieved the highest $F_1$ score and the lowest FDP. Additionally, Table \ref{tbl:synthetic_cutoff_FDR} indicates that the obtained FDP closely matched the target FDR (i.e., $\alpha$) for all methods except Raw-NaiveSort and Raw-HierRank. {\color{black} Furthermore, results in Table S4 in Supplement D.4 show that the $F_1$ score obtained on the new testing dataset nearly matched the highest achievable $F_1$ score on the validation dataset.} These results demonstrate that the strategy described in Section \ref{sec:cutoff_selection} successfully produced a satisfactory cutoff, as expected.

\begin{table}[tp]
\caption{{\color{black}
    Observed False Discovery Proportion (FDP) and Discovery proportion (d.p.)
    on the synthetic testing dataset with the cutoff $\hat{\kappa}$  chosen to target an FDR of $\alpha \times 100\%$, for $\alpha$ values of 0.01, 0.05, 0.1, and 0.2.} The highest values in each d.p. column are highlighted in bold. All values are expressed as percentages.
}\label{tbl:synthetic_cutoff_FDR}  
  \centering  
\begin{tabular}{l|ll|ll|ll|ll}
\hline
\multicolumn{1}{r|}{Target $\alpha \times 100$}                            & \multicolumn{2}{l|}{1.0}                         & \multicolumn{2}{l|}{5.0}                         & \multicolumn{2}{l|}{10.0}                        & \multicolumn{2}{l}{20.0}                          \\ \hline
Method                                               & \multicolumn{1}{l}{d.p.}          & FDP          & \multicolumn{1}{l}{d.p.}          & FDP          & \multicolumn{1}{l}{d.p.}          & FDP          & \multicolumn{1}{l}{d.p.}           & FDP           \\ \hline\hline
Raw-NaiveSort                                         & \multicolumn{1}{l}{0.002}        & 0.0          & \multicolumn{1}{l}{0.002}        & 0.0          & \multicolumn{1}{l}{0.026}        & 31.3         & \multicolumn{1}{l}{0.032}         & 37.5          \\ 
Raw-HierRank                                          & \multicolumn{1}{l}{0.002}        & 0.0          & \multicolumn{1}{l}{0.005}        & 5.1          & \multicolumn{1}{l}{0.1}          & 9.5          & \multicolumn{1}{l}{1.0}           & 19.5          \\ \hline
ClusHMC-vanilla                                       & \multicolumn{1}{l}{0.007}        & 0.0          & \multicolumn{1}{l}{0.007}        & 0.0          & \multicolumn{1}{l}{3.5}          & 9.7          & \multicolumn{1}{l}{7.5}           & 19.2          \\ 
ClusHMC-bagging                                       & \multicolumn{1}{l}{0.002}        & 0.0          & \multicolumn{1}{l}{2.2}          & 5.1          & \multicolumn{1}{l}{4.6}          & 9.6          & \multicolumn{1}{l}{8.4}           & 19.6          \\ \hline
HIROM-hier.ranking                                    & \multicolumn{1}{l}{0.02}         & 0.0          & \multicolumn{1}{l}{2.9}          & 5.1          & \multicolumn{1}{l}{6.2}          & 9.7          & \multicolumn{1}{l}{10.3}          & 19.0          \\ 
HIROM-hier.Hamming                                    & \multicolumn{1}{l}{0.3}          & 3.5          & \multicolumn{1}{l}{4.6}          & 5.1          & \multicolumn{1}{l}{6.3}          & 9.5          & \multicolumn{1}{l}{8.3}           & 19.5          \\ \hline
\Method\/NaiveSort-full & \multicolumn{1}{l}{\textbf{2.7}} & 0.6 & \multicolumn{1}{l}{\textbf{5.8}} & 4.4 & \multicolumn{1}{l}{\textbf{8.3}} & 9.6 & \multicolumn{1}{l}{11.5} & 19.4 \\ 
\Method\/HierRank-full  & \multicolumn{1}{l}{\textbf{2.7}} & 0.6 & \multicolumn{1}{l}{\textbf{5.8}} & 4.4 & \multicolumn{1}{l}{\textbf{8.3}} & 9.6 & \multicolumn{1}{l}{\textbf{11.7}} & 19.5 \\ \hline
\end{tabular}
\end{table}

\begin{table}[tp]
  \caption{
  $F_1$ score on the synthetic testing dataset with the cutoff $\hat{\kappa}$  chosen to maximize the $F_1$ score on the validation dataset.  
  The corresponding FDP and d.p. are also reported. 
  The lowest FDP and the highest $F_1$ score are shown in bold. All values are percentages.}\label{tbl:synthetic_cutoff_fmax}  
  \centering
\begin{tabular}{l|ccc}
\hline
Method                                                & d.p.  & FDP  &  $F_1$ score  \\ \hline\hline
Raw-NaiveSort                                         & 98.9 & 86.7          & 29.0                         \\ 
Raw-HierRank                                          & 41.1 & 81.0          & 23.3                         \\ \hline
ClusHMC-vanilla                                       & 14.9 & 37.5          & 64.3                         \\ 
ClusHMC-bagging                                       & 11.6 & 27.2          & 66.5                         \\ \hline
HIROM-hier.ranking                                    & 13.6 & 27.6          & 73.2                         \\ 
HIROM-hier.Hamming                                    & 14.9 & 32.0          & 71.9                         \\ \hline
\Method\/NaiveSort-full & 12.8 & \textbf{23.9} & 74.8                         \\ 
\Method\/HierRank-full & 12.8 & \textbf{23.9} & \textbf{74.9}      \\ \hline
\end{tabular}
\end{table}

\subsection{Results on Disease Diagnosis Dataset}\label{sec:disease_example}
For the disease diagnosis dataset, we followed the same training procedure as \citet{huang2010} to obtain the binary Bayesian classifiers. Specifically, we used leave-one-out cross-validation (LOOCV) to calculate the Bayesian classification scores. To ensure a fair comparison, all competing methods were executed using the same LOOCV approach. For this dataset, we include results for all three versions of the mLPR estimation procedure.

For each method, we plotted the PR curve for the resulting ranking. As shown in Figure \ref{disease_diag_PR}~(a), {\color{black} among the three methods for deriving $\widehat{mLPR}$, the independence (indpt) approximation outperformed both the neighborhood (nbh) approximation and the full version.  
Despite the superior theoretical performance of the full version,
the independence approximation demonstrated greater practical robustness in this case. This result is attributed to the difficulty of appropriately incorporating dependency when estimating mLPR from a dataset with a low signal-to-noise ratio and a small sample size.} Additionally, HierRank produced a better ranking than NaiveSort for each version of the mLPR estimation. This corroborates our previous conclusion that HierRank provides a better ranking than NaiveSort when the input scores are deficient (e.g., the $\widehat{mLPR}$ values are imperfect).

As shown in Figure \ref{disease_diag_PR}~(b), \Method\/HierRank-indpt performed better than all of the other methods.  Furthermore, it performed noticeably better when the precision rate was high and the recall value was low (i.e., in the initial portion of the ranking). Specifically, our method outperformed C-HMCNN at high precision levels (i.e., when recall is approximately less than 0.125 and precision is close to 1), although C-HMCNN surpassed our approach when the recall approximately exceeded 0.4. This result suggests that our method may be more suitable for applications such as disease diagnosis, where the accuracy of the top decisions (or high precision) is more critical.

\begin{figure}[tp]
  \centering
  \begin{minipage}{0.49\linewidth}
    \centering
   \includegraphics[width=\linewidth]{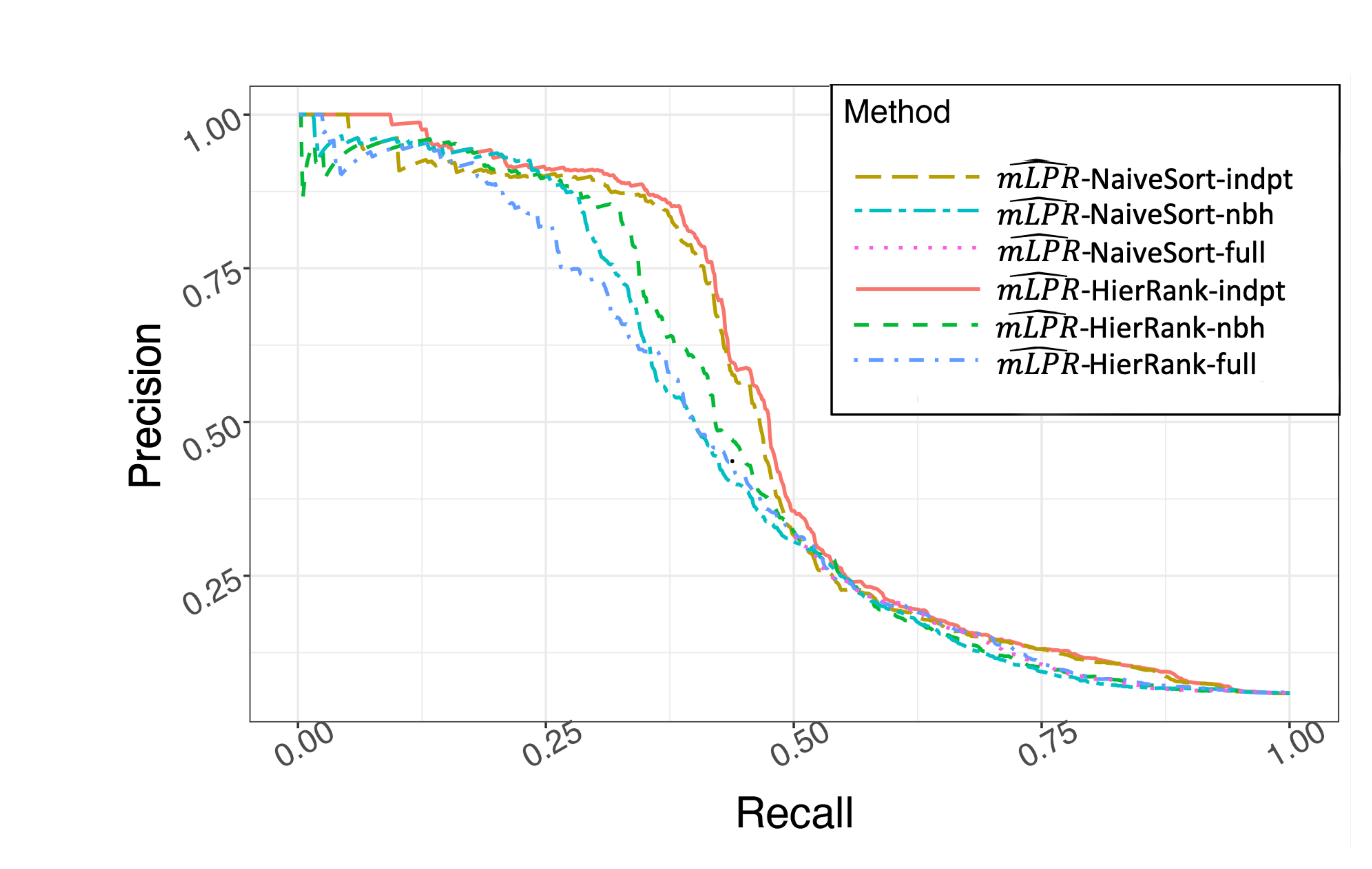}
    \subfloat[Variants of \Method\/based methods]{\hspace{\linewidth}}    
  \end{minipage}
  \begin{minipage}{0.49\linewidth}
    \centering
    \includegraphics[width=\linewidth,height=0.7\linewidth]{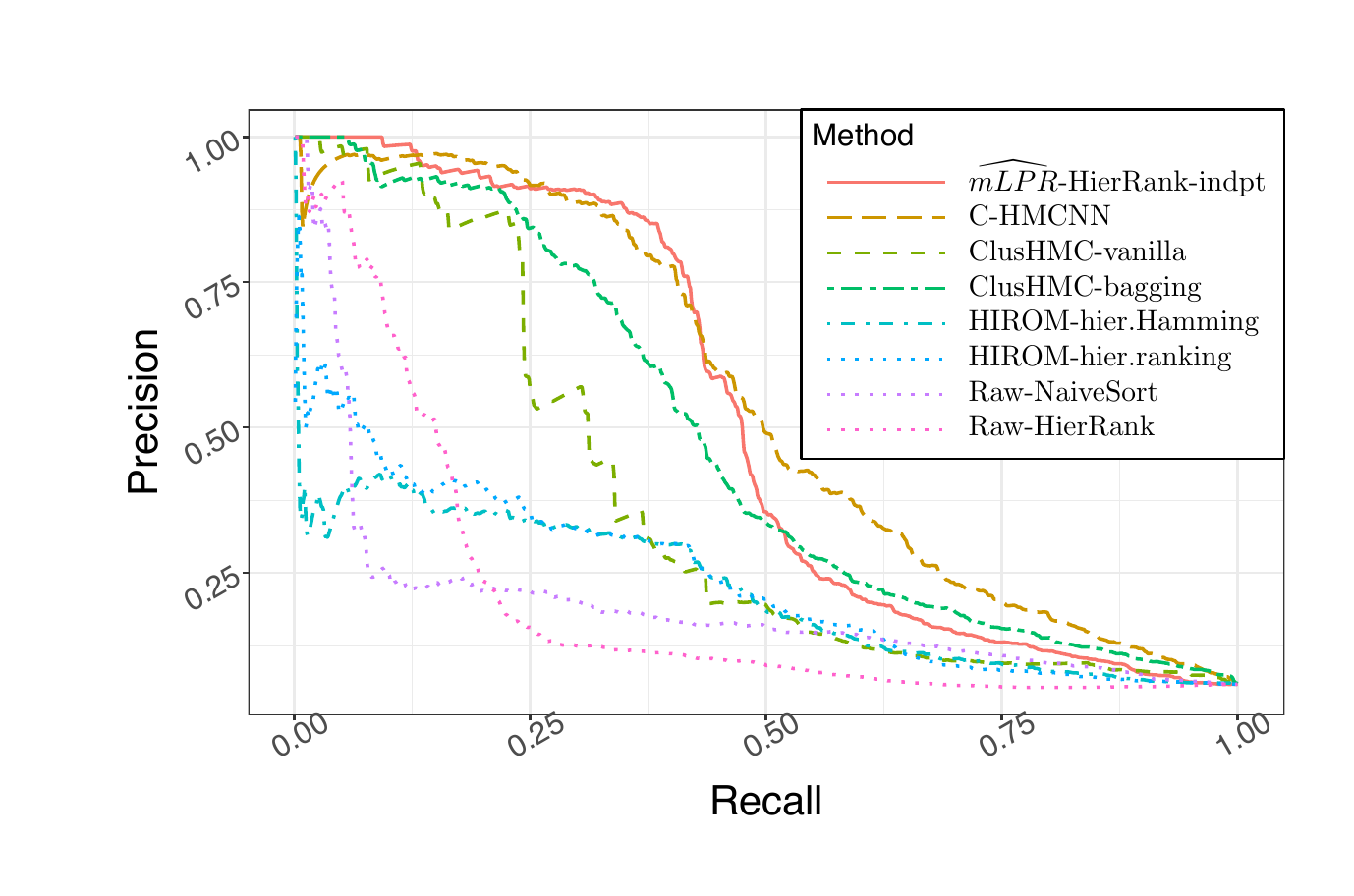}
    \subfloat[All methods]{\hspace{.5\linewidth}}    
  \end{minipage}
  \caption{Precision--recall curve for several classifiers applied to the disease diagnosis dataset of \citet{huang2010}. } \label{disease_diag_PR}
\end{figure}

\subsection{Results on RCV1v2 Dataset}\label{sec:rcv_example}
For the RCV1v2 dataset, we split the training dataset into two subsets, one for each stage of our method. Specifically, we trained class-wise support vector machines (SVMs) on the first training subset. The classifier scores output by these SVMs on the second training subset were then used to train models for mLPR estimation. This procedure was implemented to mitigate an overfitting issue that arose when both stages were trained on the same dataset.
We discuss this further in Supplement D.6 for interested readers.

Using the above strategy and following the procedure (without considering a validation set) described in Section \ref{sec:cutoff_selection}, we computed the precision and recall values for \Method\/HierRank and the competing methods on the RCV1v2 testing dataset, as shown in Table \ref{tbl:rcv_result}.
For $\kappa \leq 0.1$, the Raw-based and Clus-HMC-based methods performed worse than the HIROM variants and \Method\/-based methods. For $0.2 \leq \kappa \leq 0.3$, \Method\/-based methods outperformed all other methods. Results for $\kappa > 0.3$ are omitted because all precision values fell below $0.1$. Furthermore, the neighborhood approximation and the full version of mLPR estimation produced similar results, regardless of the sorting method used, and both were superior to the independence approximation.
These findings suggest that the neighborhood approximation is sufficient for achieving accurate mLPR estimation on this dataset. 


\begin{table}[tp]
  \caption{
  Recall and precision (prec) values on the RCV1v2 testing dataset. Here, $\kappa$ := \# predicted positives / \# all events. The highest values in each column are shown in bold. All values are percentages.
  }\label{tbl:rcv_result}  
  \centering
\begin{tabular}{l|ll|ll|ll|ll}
\hline
\multicolumn{1}{r|}{$\kappa~$}                                    & \multicolumn{2}{l|}{0.05}                          & \multicolumn{2}{l|}{0.1}                           & \multicolumn{2}{l|}{0.2} & \multicolumn{2}{l}{0.3}                                                   \\ \hline
                                                       & \multicolumn{1}{l}{recall}        & prec          & \multicolumn{1}{l}{recall}        & prec          & \multicolumn{1}{l}{recall}        & prec          & \multicolumn{1}{l}{recall}        & prec          \\ \hline\hline
Raw-NaiveSort                                          & \multicolumn{1}{l}{4.0}           & 2.5           & \multicolumn{1}{l}{5.3}           & 1.7           & \multicolumn{1}{l}{6.9}           & 1.1           & \multicolumn{1}{l}{8.5}           & 0.9                   \\
Raw-HierRank                                           & \multicolumn{1}{l}{7.3}           & 4.6           & \multicolumn{1}{l}{11.1}          & 3.5           & \multicolumn{1}{l}{17.5}          & 2.8           & \multicolumn{1}{l}{23.6}          & 2.5          \\ \hline
ClusHMC-vanilla                                        & \multicolumn{1}{l}{68.5}          & 43.2          & \multicolumn{1}{l}{80.0}          & 25.2          & \multicolumn{1}{l}{88.2}          & 13.9          & \multicolumn{1}{l}{90.8}          & 9.5           \\ 
ClusHMC-bagging                                        & \multicolumn{1}{l}{72.5}          & 45.6          & \multicolumn{1}{l}{83.7}          & 26.4          & \multicolumn{1}{l}{92.0}          & 14.5          & \multicolumn{1}{l}{95.7}          & 10.0          \\ \hline
HIROM-hier.ranking                                    & \multicolumn{1}{l}{77.1}          & 48.6          & \multicolumn{1}{l}{80.0}          & 25.2          & \multicolumn{1}{l}{85.9}          & 13.5          & \multicolumn{1}{l}{89.1}          & 9.4          \\ 
HIROM-hier.Hamming                                    & \multicolumn{1}{l}{75.4}          & 47.5          & \multicolumn{1}{l}{88.8}          & \textbf{28.0} & \multicolumn{1}{l}{91.7}          & 14.4          & \multicolumn{1}{l}{93.8}          & 9.8                   \\ \hline
\Method\/NaiveSort-indpt & \multicolumn{1}{l}{74.9}          & 47.2          & \multicolumn{1}{l}{85.8}          & 27.0          & \multicolumn{1}{l}{92.4}          & 14.5          & \multicolumn{1}{l}{94.4}          & 9.9          \\ 
\Method\/NaiveSort-nbh   & \multicolumn{1}{l}{77.8}          & 49.0          & \multicolumn{1}{l}{88.8}          & \textbf{28.0} & \multicolumn{1}{l}{93.5}          & 14.7          & \multicolumn{1}{l}{\textbf{97.0}} & \textbf{10.2}\\ 
\Method\/NaiveSort-full  & \multicolumn{1}{l}{77.5}          & 48.8          & \multicolumn{1}{l}{\textbf{88.9}} & \textbf{28.0} & \multicolumn{1}{l}{93.9}          & \textbf{14.8} & \multicolumn{1}{l}{96.8}          & \textbf{10.2}\\ \hline
\Method\/HierRank-indpt  & \multicolumn{1}{l}{75.8}          & 47.9          & \multicolumn{1}{l}{86.6}          & 27.3          & \multicolumn{1}{l}{92.7}          & 14.6          & \multicolumn{1}{l}{95.2}          & 10.0          \\ 
\Method\/HierRank-nbh    & \multicolumn{1}{l}{\textbf{78.0}} & \textbf{49.1} & \multicolumn{1}{l}{\textbf{88.9}} & \textbf{28.0} & \multicolumn{1}{l}{\textbf{94.1}} & 14.7          & \multicolumn{1}{l}{96.5}          & \textbf{10.2}\\ 
\Method\/HierRank-full   & \multicolumn{1}{l}{77.5}          & 48.8          & \multicolumn{1}{l}{\textbf{88.9}} & \textbf{28.0} & \multicolumn{1}{l}{93.6}          & \textbf{14.8} & \multicolumn{1}{l}{\textbf{97.0}} & 10.1         \\ \hline
\end{tabular}
\end{table}

\section{Discussion}\label{sec:hmc_conclusion}
In this article, we have introduced the mLPR quantity and demonstrated that sorting the true mLPRs in descending order can optimize the HMC performance, as measured by \Objective\/, while respecting the class hierarchy. As true mLPRs are not accessible, we have provided an approach to estimate them. We have developed a ranking algorithm, HierRank, that leads to the highest $\widehat{\Objective\/}$ under the hierarchical constraint. Our method can be easily employed in various HMC applications, including disease diagnosis, protein-function categorization, gene-function categorization, image classification, and text classification. Extensive experiments have demonstrated the superior performance of this approach compared to competing methods.

We conducted a comparison of three different versions of the mLPR estimation procedure, which varies in the extent of their graph usage. We found that the full version is the preferred choice when there are ample high-quality samples, in which case NaiveSort and  HierRank produce comparable results. When the data quality is poor, we recommend using the independence or neighborhood approximation for greater robustness. In such scenarios, HierRank can ensure the hierarchy compliance and boost performance. Selecting among the three versions from a theoretical standpoint remains a promising avenue for future investigation.

Finally, there is potential for further improving our method. While the \Method\/based methods have shown good performance, they hinge on the given class-wise classifiers which are not optimized under the hierarchy. To address this, an end-to-end classification system could be developed that takes the raw data (covariates) as input and aims to maximize $\Objective\/$ given the graph hierarchy.

\begin{acks}[Acknowledgments]
	We thank Xinwei Zhang, Calvin Chi, Jianbo Chen for their suggestions on this paper. 
\end{acks}



\begin{supplement}
  \stitle{Supplementary Material of ``Ranking hierarchical classification results with mLPRs''}
  \sdescription{In \href{https://drive.google.com/drive/folders/1B7l3MWpZBVzjnljpXfF_wDUXdZHLJqDS?usp=sharing}{Supplementary Material}, we provide more details about the hit curve. We discuss HierRank from various perspectives, including the formal version of HierRank, an equivalent algorithm of HierRank, a faster version of HierRank, and an extension of HierRank to DAG. We also provide theoretical justification of the cutoff selection procedure. Finally, we provide more empirical results and proofs of theorems presented in this article.}

\end{supplement}

\bibliographystyle{imsart-nameyear}
\bibliography{hierLPR}






\end{document}